\documentclass[11pt]{article}

\usepackage[T1]{fontenc}
\usepackage{lmodern}
\usepackage[margin=1in]{geometry}
\usepackage{amsmath,amssymb,amsfonts,amsthm,mathtools}
\usepackage{graphicx}
\usepackage{tikz}
\usepackage{xcolor}
\usepackage{dsfont}
\usepackage{microtype}
\usepackage[numbers,sort&compress]{natbib}
\usepackage{hyperref}

\graphicspath{{./images/}}
\hypersetup{
  colorlinks=true,
  linkcolor=teal,
  filecolor=blue,
  citecolor=cyan,
  urlcolor=violet
}

\usepackage[noblocks]{authblk}

\newtheorem{theorem}{Theorem}
\newtheorem{lemma}{Lemma}
\newtheorem{corollary}{Corollary}

\theoremstyle{definition}
\newtheorem{definition}{Definition}
\newtheorem{example}{Example}

\theoremstyle{remark}
\newtheorem{remark}{Remark}

\newcommand{\N}{\mathbb{N}}

\newcommand{\R}{\mathbb{R}}
\newcommand{\Prob}{\mathbb{P}}
\DeclareMathOperator{\ind}{\mathds{1}}
\DeclareMathOperator{\E}{\mathbb{E}}
\newcommand{\SNE}{{\rm SNE}}
\newcommand{\Int}{{\rm int}}
\newcommand{\eq}{p^{*}}
\newcommand\restr[2]{{\left.\kern-\nulldelimiterspace #1\right|_{#2}}}

\def\r{\textcolor{red}}

\def\exp{{\rm exp}}

\def\1{\mathbf{1}}
\def\P{\mathbb{P}}

\title{Games played by Exponential Weights Algorithms}

\author[1]{Maurizio D'Andrea}
\author[2]{Fabien Gensbittel}
\author[2]{J\'{e}r\^{o}me Renault}

\affil[1]{National University of Singapore, Singapore\\
\texttt{dmaurizio\_1@nus.edu.sg}}

\affil[2]{Toulouse School of Economics,
University of Toulouse Capitole, Toulouse, France\\
\texttt{fabien.gensbittel@tse-fr.eu} \quad
\texttt{jerome.renault@tse-fr.eu}}

\date{}

\begin{document}
\maketitle

\begin{abstract}
This paper studies the last-iterate convergence properties
 of the exponential weights algorithm with constant learning rates.  We consider a repeated interaction in discrete time, where each player uses an exponential weights algorithm characterized by an initial mixed action and a fixed learning rate, so that the mixed action profile $p^t$ played at stage $t$ follows an homogeneous Markov chain. At first, we show that  the probability to play a strict Nash equilibrium at the next stage  converges almost surely to 0 or 1. Secondly, we show that the limit of $p^t$, whenever it exists, belongs to the set of {\it Nash Equilibria with Equalizing Payoffs}. Finally, we show that in strong coordination games, where the payoff of a player is positive on the diagonal and 0 elsewhere, $p^t$ converges almost surely  to one of the strict Nash equilibria.
\end{abstract}

\noindent\textbf{Keywords:} Repeated Games, Exponential Weights Algorithms, Nash Equilibrium with Equalizing Payoffs.

\medskip
\noindent\textbf{Funding:} This research has benefited from the financial support of the AI Institute ANITI, under the Grant agreements ANR-19-PI3A-0004 and ANR-23-IACL-P0002, and of the ANR Chess ANR-17-EURE-0010. J. Renault also acknowledges funding from ANR MaSDOL-19-CE23-0017-01.

\bigskip

\section{Introduction}~\\
\textbf{Motivation:}
Many machine learning algorithms used for prediction or decision-making are designed to optimize the behavior of a  single agent facing an unknown, changing,  environment. However,  with the increasing use of these algorithms in various fields and the complexity of the problems at hand, interaction between these algorithms, designed for an  agent unconscious of other players,  have become common. This raises a natural question: where will these interactions lead?

Our paper contributes to the  literature  on learning algorithms in games. We have in mind players who are engaged in a game without being conscious that they are playing this  game. Each player simply considers the joint action of the other players as the environment (or « the market »).  We choose to assume that each player uses a Vanilla Exponential Weights Algorithm, which is a very simple algorithm used as a base brick in many situations. As the players are different, there is no reason to assume that they use the same learning rate, nor that they   start together at a common stage 0, or with the same uniform  initial weights.   The update of the weights of the actions of a player  is  performed using the realized actions of the other players, which we believe is natural in practice (we observe realizations of the possibly mixed actions of the other players, not how they flipped  their coins). This leads to the study of a stochastic process which is a Markov chain, and new challenging mathematical  problems that we partly solve.

Precisely, we consider  an underlying finite game and analyze the day to day behavior of the exponential weights (EW) algorithm (\cite{NR0,NR2,NR3,NR4}) with constant learning rates, when applied independently by all the players. This  algorithm   is one of the most popular and widely studied algorithm with applications in various domains and contexts: computational geometry, optimization, operations research, online statistical decision-making, machine learning (we refer to \cite{CBL,Shai,Regret} for a general account on the subject). The idea behind EW is the following: at each stage $t\geq 0$, each player $i$ chooses an action at random according to their mixed strategy $p_i^t$ and then observe their random vector payoff (one coordinate for each possible action of that player)  which depends on the realized actions of the other players. The probability to play each action at stage $t+1$ for player $i$ is proportional to the exponential of the sum of the payoffs he would have obtained by playing that same action at all past stages multiplied by a learning rate $\eta_i$. 
This work belongs to the ``behavioral'' part of learning in games, we model agents  choosing actions in a simple way and explore if we get convergence to something. Another, flourishing part of the literature, considers a given underlying game and designs algorithms which hopefully will converge to a Nash equilibrium of this game. In this ``computational'' or ``programming'' part of  learning, it is natural and common to  use smarter  algorithms designed for games, such as optimistic variants of EW or Gradient Descent Ascents. It is also common to assume that players, at each step of the algorithm, will react to the mixed actions, as opposed to their realizations, played by the other players. 

This  leads to the study of   deterministic sequences of mixed actions (as opposed to our Markov chains over mixed actions). 
Such an approach requires to assume that players know that they are engaged in a game and use algorithms  that have been specifically designed for games.
In contrast to this important branch of the literature, we assume that the  players are using algorithms designed for decision-making without being conscious that  they are in a game, and analyze the asymptotic behavior of their strategies. 
Positive convergence results are more difficult to obtain when naive algorithms are used and when players observe realized actions compare to mixed strategies, and we consider that understanding when they arise is an important and interesting question.

\textbf{Related literature:}  
A commonly used performance criterion is regret minimization, first introduced by Hannan \cite{Hannan}. 
A no-regret strategy ensures that, looking at the history of the strategies played, there is no single action that, if consistently chosen, would have provided a better payoff. EW with appropriate time-dependent and decreasing learning rates is a no regret strategy (see e.g. Corollary 4.3 page 73 in \cite{CBL}). A classical result by Hart and Mas-Colell \cite{HartMasColell} states that if every player follows a no-regret strategy, the empirical average distribution almost surely converges to the set of coarse correlated equilibrium distributions. However, this result holds only for time-average and the set of coarse correlated equilibria can include strategies that are far from rational, making the final prediction of empirical convergence to these equilibria rather mild.

Another possible approach to EW with decreasing learning rates, which applies to a large family of algorithms, is known as the stochastic approximation method and originates from the work of Benaïm \cite{Ben}. Interpreting the learning rates as time-steps, one obtains a discrete-time system whose trajectories can be compared asymptotically to the solutions of a differential equation. Using this method, it was proven in  \cite{LerningPotentialgames} that EW with appropriate decreasing learning rates (as well as other variants of this algorithm) almost surely converges to Nash equilibria in potential games. As explained in section 3.2 of \cite{LerningPotentialgames}, a crucial point in the analysis is that for a potential game, the payoff is increasing along the solutions of  the limit differential equation,
 which is the replicator dynamics of Taylor and Jonker \cite{taylor1978evolutionary}. In contrast, in the case of constant learning rates, for potential games, the payoff is not in general a sub-martingale, and thus not increasing in expectation (see Remark \ref{remark_potential}). 

EW belongs to a larger class of dynamics called FTRL (follow the regularized leader) which has been extensively studied. In \cite{giannou2021}, it is proven that for such a dynamic with varying learning rates and general observation framework (including the one we consider in the present paper and the bandit case), strict Nash equilibria are asymptotically stable while non-strict Nash equilibria cannot be stable (let us also mention \cite{vlatakis2020} for similar results for continuous-time learning dynamics). The class of dynamics considered in \cite{giannou2021} contains EW with decreasing learning rates and  the non-stability result also applies to EW with constant learning rates. We provide a precise comparison in Remark \ref{remark_ftrl}, showing that our results, while having a similar flavor, provide stronger statements compared to the stability and non-stability results of \cite{giannou2021} in the particular case of EW with constant learning rates, as well as new qualitative properties.

In \cite{MWUstochastinZSG}, EW with constant learning rates is analyzed in zero-sum games, using a Markov chain approach similar to the one we use. It is  proven in \cite{MWUstochastinZSG} that when the game is zero-sum and admits a completely mixed Nash equilibrium with rational coefficients, then 
for any compact $B$ contained in the relative interior of the space of mixed strategy profiles, for the Markov chain $p^t$ of mixed strategy profiles the expected return time to $B$  is infinite,  and  the average time spent by $p^t$ in B  converges to zero almost surely, which can be interpreted as a weak form of convergence to the boundary for the Markov chain.
Almost sure convergence to the boundary is also proven in \cite{MWUstochastinZSG} for zero-sum games having a unique Nash equilibrium in the boundary. 
These results are then used to prove that in generic games the limit distributions of time-average strategies are mixtures of pure strategy profiles. 
Our results of sections \ref{sec:strictNE} and \ref{sec:neep} also apply to zero-sum games as a particular case but have a different nature although we do also prove convergence to the boundary under different assumptions. Precisely, we prove (see Corollary \ref{cor1}) that for a general game, possibly non zero-sum, if a strict Nash equilibrium exists, the Markov chain $p^t$ converges almost surely to the boundary. As the existence of a strict Nash equilibrium and that of a fully mixed Nash equilibrium in a zero-sum game are mutually exclusive assumptions, together with the fact that we do not analyze the asymptotic behavior of time-average strategies, our results have very little overlap  with those of \cite{MWUstochastinZSG}.

A variant of EW with constant learning rates but with added noise was considered in \cite{WSNE}. The algorithm was analyzed through a comparison with the solutions of a continuous-time dynamical system, as in the stochastic approximation method, and it was shown that the players play most of the time an approximate equilibrium.

Another stream of the literature analyzed the deterministic dynamical system associated to EW (whereas the algorithm we analyze is sometimes referred to as stochastic EW). This amounts to consider that each player observes the expectation of the vector payoff, which leads to a deterministic discrete-time evolution. In congestion games, it was shown in \cite{palaiopanos2017multiplicative} that the constant learning rate version of this algorithm may lead to cycles and chaotic behavior, a phenomenon which was already analyzed for continuous-time dynamics in zero-sum games in \cite{mertik2018cycles}. This dynamic is also studied in \cite{APFS22} where it is shown to converge to the set of Nash equilibria in potential games. Several works consider variants of this deterministic dynamic with a focus on the regret analysis, see e.g. \cite{CP20}, \cite{DFG21} and \cite{GPD20}.

Our model has three important features that may be discussed in regard of the existing literature on learning: 1) we focus on a particular algorithm, namely EW, 2) we consider that players use constant learning rates (that may be different across players), and 3) we assume that players observe their random vector payoff (one coordinate for each possible action of that player)  which depends on the realized actions of the other players. 
For the first point, our motivation is that EW is one of the most popular and widely studied algorithm that can be used as a benchmark in the analysis of more general or different models. However, it is important to note that other choices could be meaningful, such as best or better reply dynamics, and that several algorithms have already been considered, let us cite for example a numerical analysis of $Q$-learning algorithms in repeated games to study collusion in prices in \cite{calvano}, or the convergence of fictitious play in zero-sum and identical payoff ergodic stochastic games in \cite{larakibaudin1}.    
Concerning the second point, in addition to the simplicity of the model (the EW algorithm with constant learning rate being used as a basic brick in many  algorithms), a reason to focus on constant learning rates is that in practice agents often tune the learning rate by starting with a fixed  learning rate, and reducing it a finite number of times, our analysis corresponding thus to the last phase of such a process. This idea of learning rate decay already appears in the introduction of \cite{fort} to justify the asymptotic analysis of algorithms with constant rates and is still discussed in more recent works such as in \cite{jordan} (see figure 1 and the discussion below), although there is a large debate on the optimal tuning of general time-dependent learning rates for general algorithms. Note also that gradient algorithms with constant learning rates are used as a benchmark in the analysis of training of deep neural networks in \cite{arora}.
Concerning the third point, as we explained,  we have in mind agents interacting  without being conscious of the game they are playing, and thus it is more natural to assume that they observed realized actions rather than mixed actions, that they interpret as the environment. 
However, the assumption that they observe for every action, its payoff against the realized profile of the opponents (also called counterfactual payoffs) is restrictive in some contexts.  
A different and important case is the bandit setting, where a player only observes the realized payoff of the action actually played. A natural adaptation of EW in that framework is the EXP3 algorithm (“exponential-weight algorithm for exploration and exploitation”) and its variants; see \cite{auer}. In EXP3, at each stage the player mixes the uniform distribution with a strategy that plays proportionally to current weights, and only the weight of the action played is updated from one stage to the next. 
Convergence of this algorithm with decreasing learning rates has been proved in \cite{LerningPotentialgames} for potential games.
Extending our results to the bandit setting is left for future research.

\textbf{Our contribution:}
To our knowledge, although there is a large literature on the EW algorithm, its variants and generalizations, convergence  of the actual sequence of strategies generated by the (stochastic) EW algorithm with constant learning rates remains essentially an open problem, that we propose to address.
First, we show  two  structural results for the Markov chains induced by this algorithm.
We prove that  the probability to play at the next stage a strict Nash equilibrium  converges almost surely to 0 or 1. Moreover  we define  the notion of Nash Equilibrium with Equalizing Payoffs (NEEP) and show that the limit of the  EW process, whenever it exists, belongs to this class of   equilibria.
Second, we prove that in strong coordination games, where the payoff of a player is positive on the diagonal and 0 elsewhere, the sequence of mixed strategy profiles converges almost surely  to one of the strict Nash equilibria. 

The paper is organized as follows:
In Section \ref{sec:model}, we introduce the exponential weights algorithm with constant learning rates, and the homogeneous Markov chain on the space of mixed strategy profiles it induces. 
We illustrate with a simple example that this Markov chain does not always converge almost surely. 
In Section \ref{sec:strictNE}, we study games with strict Nash equilibria and prove that the probability to play a given strict Nash equilibrium at every stage is positive and a continuous function of the starting distribution. 
Furthermore, we show the equivalence between the convergence of the induced dynamics to a strict Nash equilibrium and the stronger statement that eventually players will always play this strict Nash equilibrium. 
These results imply that the strategies induced by EW converge almost surely to the boundary of the set of mixed strategy profiles, as soon as a strict Nash equilibrium exists.
In particular, we characterize which part of the boundary is reachable by the dynamics, and as a corollary, we show that the probability of playing a strict Nash equilibrium at stage $t$ almost surely converges to 0 or 1.
In Section \ref{sec:neep}, we consider Nash equilibrium with equalizing payoffs (NEEP). 
It is well known that two actions in the support of a Nash equilibrium have the same expected payoff under it; in a NEEP, this equality between payoffs holds almost surely. 
We then prove that the existence of a NEEP is a necessary condition for convergence in EW with constant learning rate, i.e., the probability that the induced dynamics converge to a mixed strategy profile that is not a NEEP is equal to zero. 
This implies that in games like matching pennies, there is no hope for convergence with EW. Examples show that the existence of a NEEP is   not a sufficient condition for convergence. 
In Section \ref{sec:coord}, we consider the class of strong coordination games, where all the players have the same action set, and the payoff  of each  player is positive whenever all the players play the same action, and 0 otherwise. 
We prove that the Markov chain induced by EW almost surely converge to one of the strict Nash equilibria.
Finally, we conclude with some simulations in Section \ref{sec:simulations} and  open questions in Section \ref{sec:remarks}.

\section{Model}
\label{sec:model}
Throughout the paper we fix a finite normal form game $G=(N, (A_i)_{i\in N}, (u_i)_{i\in N})$, where $N=\{1,...,n\}$ is the set of  players. For each  $i$ in $N$, $A_i$ is the finite set of actions of player $i$ and $u_i:  \prod_{j=1}^n A_j \to \R$ is the payoff function of player $i$. 
\vspace{0.30cm}

\noindent{\bf Notations:} $A=\prod_{j=1}^n A_j$ is the set  of actions profiles  of all players  and for $i$ in $N$,   $A_{-i}$ is  the set $\prod_{j\in N \backslash \{i\}} A_j$ of action profiles of all  players different  from   $i$.   If $a=(a_j)_{j \in N}$ and $i\in N$, $a_{-i}$ denotes $(a_j)_{j \in N\backslash\{i\}}$. As usual, we write  $\Delta(A_i)$ to represent the set of mixed actions of player $i$ (probability distributions over $A_i$).  $\Delta=\prod_{i \in N} \Delta(A_i)$ is the set of mixed action profiles and the payoff functions are extended linearly to $\Delta$. We denote by $\Int(\Delta)=\{p \in \Delta, \forall i \in N \; \forall a_i \in A_i,   p_i(a_i)>0\}$ the relative interior of $\Delta$. {We identify $p \in \Delta$ with the product probability over $\prod_{i \in N} A_i$ defined by $p(a)=\prod_{i \in N}p_i(a_i)$. We also identify each pure action $a_i \in A_i$ with the Dirac mass at $a_i$.}
\vspace{0.30cm}

The game $G$ will be played at stages $t=0,1,...$, and we assume that each player $i$ uses an Exponential Weights Algorithm characterized by a learning rate $\eta_i>0$ and positive  initial weights $(w_i^0(a_i))_{a_i\in A_i}$. This algorithm unfolds  as follows:  
\begin{itemize}
\item each action $a_i$ will have a positive  weight $w_i^t(a_i)$ for each stage $t$,
\item at  each   stage $t$, player $i$ plays proportionally to the  current weights, i.e. plays every action $a_i$ in $A_i$ with probability $\frac{w^{t}_i(a_i)}{w_i^t}$, where  $w_i^t:=\sum_{a_i\in A_i} w_i^t(a_i)$ is the current sum of weights,
\item at the end of each stage $t$, player $i$ observes the action profile $a_{-i}^t$ in $A_{-i}$  played by the other players at $t$, and updates its weights by:
\[\forall a_i \in A_i,\; w_i^{t+1}(a_i)= w_i^t(a_i) \exp(\eta_i u_i(a_i, a_{-i}^t)).\]
\end{itemize}
Notice that all weights remain positive, so that every pure action is played with positive probability at every stage. 

We write $p_i^t(a_i)= \frac{w^{t}_i(a_i)}{w_i^t}$, so that $p_i^t= (p_i^t(a_i))_{a_i\in A_i}\in \Delta(A_i)$ is the mixed action played by player $i$ at stage $t$, and $p^t=(p_i^t)_{i\in N}\in \Delta$ is the mixed action profile played at stage $t$. 
As soon as $t\geq 1$, $p^t$ depends on the realized  actions   played by the other players at previous stages, so $(p^t)_t$ is not a deterministic sequence but a stochastic process.  For each $t\geq 0$ we denote by ${\cal F}_t$ the $\sigma$-algebra generated by the first $t$ actions $a^0,..., a^{t-1}$ in $A$. Then $p^t$ is ${\cal F}_t$-measurable for every $t$. We denote by ${\cal F}_\infty$
the $\sigma$-algebra generated by all  the  ${\cal F}_t$'s, then the sequence $(p^t)_t$ is ${\cal F}_\infty$-measurable.

Simple computations show that for each player $i$, stage $t$ and action $a_i$ in $A_i$:
 $$w^t_{i}(a_i)=w^0_{i}(a_i) e^{\eta_i\sum_{k=0}^{t-1}u_i(a_i,a_{-i}^k)},$$

\begin{equation}\label{eq1}p^t_{i}(a_i)  =\frac{w^0_{i}(a_i) e^{\eta_i\sum\limits_{k=0}^{t-1}u_i(a_i,a_{-i}^k)}}{\sum\limits_{b_i\in A_i}w^0_{i}(b_i) e^{\eta_i\sum\limits_{k=0}^{t-1}u_i(b_i,a_{-i}^k)}}=\frac{p^0_{i}(a_i)}{\sum\limits_{b_i\in A_i} p^0_{i}(b_i)e^{\eta_i \sum\limits_{k=0}^{t-1}(u_i(b_i,a^k_{-i})-u_i(a_i,a^k_{-i}))}},  
\end{equation}

\begin{equation}\label{eq2}p^t_{i}(a_i) =\frac{w^{t-1}_{i}(a_i)e^{\eta_i u_i(a_i,a_{-i}^{t-1})}}{\sum\limits_{b_i\in A_i}w^{t-1}_{i}(b_i)e^{\eta_i u_i(b_i,a_{-i}^{t-1})}} = \frac{p^{t-1}_{i}(a_i)}{\sum\limits_{b_i\in A_i} p^{t-1}_{i}(b_i)e^{\eta_i (u_i(b_i,a^{t-1}_{-i})-u_i(a_i,a^{t-1}_{-i}))}}.
\end{equation}

We have the fundamental property:
\vspace{0.25cm}

\noindent {\bf{Claim:}}   $(p^t)_{t\geq 0}$ is  an homogeneous Markov chain taking values in $\Delta$. 
\vspace{0.25cm}
 
With our assumptions, $p^0$ belongs to the relative interior  of $\Delta$. However for convenience we will sometimes also consider initial mixed profiles $p^0$  in the boundary of $\Delta$, the definition of the Markov chain in this case being entirely similar. For $p^0$ in $\Delta$, we denote by $\P_{p^0}$ (or sometimes simply $\P$)  the induced law on the sequences $(a^t)_t$ and $(p^t)_t$. 
Notice that if $p^0_i(a_i)=0$ for some player $i$ and action $a_i$ in $A_i$, then $\P_{p^0}$-almost surely $p^t_i(a_i)=0$ for all $t$.

The almost sure convergence of $(p^t)_t$ is not true without additional assumptions. We conclude this section with  a specific  example where this convergence fails.  

\begin{example} \label{ex1111}
Consider  the following bimatrix game:
\[
    \bordermatrix{
     & L & R  \cr
    T & (1,1) & (1, 0) \cr
    B & (1,0) & (1,1)
    }.\]
Notice that Player 1's payoff is always 1.  Here the Nash equilibria are the elements of $E_1 \cup E_2 \cup E_3$, with $E_1=\{(xT+(1-x)B,R), x\in [0,1/2]\}$, $E_2=\{(xT+(1-x)B,L), x\in [1/2,1]\}$ and $E_3=\{(1/2 T+1/2 B,  yL+(1-y)R), y\in [0,1]\}$.  
  
    Assume that $p^0$ is the uniform distribution. Since Player 1's payoff is constantly  1, we have $p_1^t(T)=p_1^t(B)=1/2$ for every $t\geq 0$. Moreover
 $$p_2^t(L)=\frac{1}{1+e^{\eta_2 Z_t}},$$
    where $Z_t=\sum_{k=0}^{t-1}(\1_{a_1^t=B} -\1_{a_1^t=T})$  is a symmetric random walk in the set  $\mathbb{Z}$ of integers. By Polya's recurrence theorem $(Z_t)_t$  is a recurrent Markov chain, so $(Z_t)_t$ has no limit    in $\mathbb{Z} \cup\{-\infty, +\infty\}$,  and consequently $(p_2^t)_t$ does not converge. \hfill $\Box$
    
\end{example}
\vspace{0.2cm}

 Notice that the non-convergence in Example \ref{ex1111} requires that $p_0=1/2$. We will show in Example \ref{exaNC} that non-convergence may occur for any interior initial condition.  

We end this section with a lemma. If  $B\subset A$ is a subset of pure strategy profiles and $p\in \Delta$ is a mixed strategy profile, we write $p(B)=\sum_{b=(b_i)_{i\in N}\in B} \prod_{i \in N} p_i(b_i)$ for the probability under $p$ to play an element of $B$. 

\begin{lemma} \label{LevyExt} 
    Let $B\subset A$ be a subset of pure strategy profiles. For any starting distribution  $p^0$ in $\Delta$, we have $\P_{p^0}$-almost surely:
\begin{equation}\label{eqLevyExt}
\frac1{t}  \sum_{k=0}^{t-1} \left( \1_{a^k\in B}-p^k(B)\right) \xrightarrow[t \to \infty]{}0.
\end{equation}
\end{lemma}

\begin{proof} 
For $t=0,1,2,...,+\infty$, define the random variables: 
$$Z_t=\sum_{k=0}^{t-1} \1_{a^k\in B} \quad {\rm and}\quad Y_t=\sum_{k=0}^{t-1} p^k(B),$$
By  Levy's extension of Borel-Cantelli lemma  (see, for instance, Theorem 12.15 p. 124  in \cite{Williams}), we have almost surely: 
 
a)  if $Y_\infty <\infty$, then $Z_\infty<\infty$

 b)  if $Y_\infty=\infty$, then $Z_\infty=\infty$  and $\frac{Z_t}{Y_t}\xrightarrow[t \to \infty ]{}1.$

 In case a), both $\frac{1}{t} Y_t$ and $\frac{1}{t} Z_t$ converge to 0, and the result is clear. In case b), $(\frac{1}{t} Y_t)_t$ and $(\frac{1}{t} Z_t)_t$ are bounded sequences such that $\frac{\frac{1}{t} Z_t}{\frac{1}{t} Y_t}$ converges to 1, so $\frac{1}{t} Z_t-\frac{1}{t} Y_t \xrightarrow[t \to \infty]{}0.$ \hfill
 \end{proof}


\section{Strict Nash Equilibria}
\label{sec:strictNE}
Recall that a strict  Nash equilibrium   of $G$ is a (necessarily pure)  action profile $a$ in $A$ such that:  for each $i$ in $N$ and $b_i\neq a_i$ in $A_i$, $u_i(b_i,a_{-i})<u_i(a)$.

We start with 3 lemmas, which proofs are in Appendix. 

\begin{lemma}\label{lem1} Let $a$ be a strict Nash equilibrium  of $G$, and $p^0$ in $\Delta$.

\begin{enumerate}
    
\item[a)]  If   $p^0(a)>0$   then $\mathbb{P}_{p^0}(a^t=a  \quad \forall t\geq 0)>0$, i.e. the probability to play $a$  at each stage is   positive.

\item[b)] The map $(p^0\mapsto  \mathbb{P}_{p^0}(a^t=a \quad \forall t\geq 0))$ is continuous from $\Delta$ to $\R$. 

\item[c)] $\mathbb{P}_{p^0}$-almost surely, $$\inf_{t\geq 0} \|p^t-a\|=0 \; \Longleftrightarrow(p^t\xrightarrow[t \to \infty]{}\  a)   \Longleftrightarrow \exists t_0 \; \forall t\geq t_0 \; a^t=a.$$
\end{enumerate}

\end{lemma}

\vspace{0.15cm}

The assumption that $a$  is   a strict NE is  thus sufficient for Lemma \ref{lem1} to hold. The next result  shows that  the assumption is also necessary    for a) and b).  

\begin{lemma} \label{lem2} Assume that $a$ in $A$ is not a strict NE of $G$. Then for all $p^0$ in $\Int(\Delta)$, $$\mathbb{P}_{p^0}(\exists t_0, \forall t\geq t_0, a^t=a  \quad \forall t\geq 0)=0,$$ 

 and the map $(p^0\mapsto  \mathbb{P}_{p^0}(a^t=a \quad \forall t\geq 0))$ is not continuous on  $\Delta$. 
\end{lemma}

Building on Lemma \ref{lem1}, we can prove the following result.
    
\begin{lemma} \label{lem25}  Let $a$ be a strict Nash equilibrium  of $G$, and   $K$ be a compact subset of $\Delta$ containing  $p^0$ and  such that $\min_{p \in K} p(a)>0$.  

 If $a\notin K$, then $\mathbb{P}_{p^0}$-almost surely  there exists $t_0$ such that $\forall t\geq t_0, p^t\notin K$.  

 If $a\in K$, then $\mathbb{P}_{p^0}$-almost surely: $\left(\exists t_0, \forall t\geq t_0,  a^t=a \right)\; {\rm   or} \; \left( \exists t_0, \forall t\geq t_0, p^t\notin K\right).$ 
\end{lemma}
      
We can now prove our first theorem.

\begin{theorem}\label{thm1}
Let $a$ be a strict Nash equilibrium of $G$, and $p^0\in \Delta$. Then $\mathbb{P}_{p^0}$-almost surely:
$$\left(\exists t_0, \forall t\geq t_0,  a^t=a \right)\; {\rm   or} \left( p^t(a) \xrightarrow[t \to \infty]{} 0\right).$$
\end{theorem}

\begin{proof} 
For every $\varepsilon>0$, consider $K_\varepsilon=\{p \in \Delta, p(a)\geq \varepsilon\}.$ Then $\mathbb{P}_{p_0}$ almost surely: - either for every $\varepsilon\in (0, p_0(a))$, the set $\{t\geq 0, p^t\in K_\varepsilon\}$ is finite, and $p^t(a) \xrightarrow[t \to \infty]{} 0$; - or there exists $\varepsilon \in (0, p_0(a))$ such that  the set $\{t\geq 0, p^t\in K_\varepsilon\}$ is infinite, and  \ref{lem25} implies that $\exists t_0, \forall t\geq t_0,  a^t=a $. \hfill
\end{proof}
\vspace{0.25cm}

 \begin{example} \label{exa1}
   Let $G$ be the 2-player game with payoff matrices: 
     \[
    \bordermatrix{
     & L & R  \cr
    T & (1,1)^{\r{\bullet}} & (0, 0) \cr
    B & (0,0) & (1,1)^{\r{\bullet}}
    }.
    \]
 Here we have 2 strict NE, identified with a red dot: $(T,L)$ and $(B,R)$.   Fix $p^0$ in $\Int(\Delta)$, and let us prove that almost surely:
 \begin{equation} \label{eq17} \exists t_0, \forall t\geq t_0,  a^t=(T,L) \;{\rm  or}\;  \exists t_0, \forall t\geq t_0,  a^t=(B,R).\end{equation}
 
 For the sake of contradiction, assume that with $\mathbb{P}_{p^0}$-positive probability, $(\forall t_0, \exists  t\geq t_0,  a^t\neq(T,L))$ and $(\forall t_0, \exists  t\geq t_0,  a^t\neq(B,R))$. Then by Theorem \ref{thm1}, we have $p^t(T,L)  \xrightarrow[t \to \infty]{} 0$ and $p^t(B,R)  \xrightarrow[t \to \infty]{} 0$. So $p^t(\{T,R\}\cup\{B,L\}) \xrightarrow[t \to \infty]{} 1$. 
{As $p^t$ is a product probability, it implies that $\max\{ p^t(T,R), p^t(B,L) \} \xrightarrow[t \to \infty]{} 1$.}  
But {it follows directly from \eqref{eq2} that}  for $\varepsilon>0$ small, it is not possible to go from some period $t$ to the next period $t+1$   from a mixed strategy profile $p^t$ which is $\varepsilon$-close to $(T,R)$ to some $p^{t+1}$ which is $\varepsilon$-close to $(B,L)$. This implies that either $p^t(T,R)  \xrightarrow[t \to \infty]{} 1$ or $p^t(B,L)  \xrightarrow[t \to \infty]{} 1$. 
{ 
Assume  for instance that  $p^t(T,R)  \xrightarrow[t \to \infty]{} 1$. Then $p^t_1(T) \xrightarrow[t \to \infty]{} 1$ and  $p^t_2(R) \xrightarrow[t \to \infty]{} 1$. 
However we have for each $t$, by equation (\ref{eq1}):
\[\frac{p_1^t(B)}{p_1^t(T)}= \frac{p_1^0(B)}{p_1^0(T)} \exp(\eta_1\sum_{k=0}^{t-1} (u_1(B,a_2^k)-u_1(T,a_2^k))).\]
As $\frac{p_1^t(B)}{p_1^t(T)}\xrightarrow[t \to \infty]{}0$, we get $\sum_{k=0}^{t-1} (u_1(B,a_2^k)-u_1(T,a_2^k))\xrightarrow[t \to \infty]{}-\infty$, so 
\[ \sum_{k=0}^{t-1} \1_{a_2^k=R} -   \1_{a_2^k=L} \xrightarrow[t \to \infty]{}-\infty.\]   
By Lemma \ref{LevyExt} we have 
\[ \frac1{t}  \sum_{k=0}^{t-1} \left( \1_{a_2^k=R}-p_2^k(R)\right) \xrightarrow[t \to \infty]{}0 \text{ and } \frac1{t}  \sum_{k=0}^{t-1} \left( \1_{a_2^k=L}-p_2^k(L)\right) \xrightarrow[t \to \infty]{}0.\]
It follows that $p^t_2(R) \xrightarrow[t \to \infty]{} 1$ implies $ \frac1{t}  \sum_{k=0}^{t-1} (\1_{a_2^k=R}-\1_{a_2^k=L})  \xrightarrow[t \to \infty]{}1$.
This is a contradiction, and repeating the argument for the case $p^t(B,L) \xrightarrow[t \to \infty]{} 1$,  (\ref{eq17}) is proved.  \hfill $\Box$
}
 \end{example}

\vspace{0,5cm}

 As a corollary of Theorem \ref{thm1}, we obtain   ``convergence to the boundary'':  if a strict NE exists, then for each $p^0$ in $\Delta$, almost surely $d(p^t, \partial \Delta) \xrightarrow[t\to \infty]{} 0$. We can strengthen  this result as follows.
 
 \begin{definition} 
  Denote by $\SNE$ the finite set of strict  Nash equilibria of $G$, and let  
    $$Z=\{p \in \Delta, p(SNE)=0\}$$ 
    be the set of mixed strategy profiles playing each  strict NE with probability 0.
 \end{definition}

\begin{corollary} \label{cor1}
For any $p^0\in \Delta$, we have  $\mathbb{P}_{p^0}$-almost surely:
 Either there exists a strict Nash equilibrium $a\in A$ and $t_0$ such that  $a^t=a$ for all $t\geq t_0$, 
 or   $d(p^t, Z)\xrightarrow[t \to \infty]{}0$. 
\end{corollary}
\begin{proof}
Assume that for each $a$ $SNE$, $\{t\geq 0, a^t\neq a\}$ is infinite. Then by Theorem \ref{thm1}, 
$$\sum_{a \in SNE} p^t(a)\xrightarrow[t \to \infty]{}0.$$ 
Since $\Delta$ is compact and $(p\mapsto \sum_{a \in SNE} p(a))$ is continuous, this implies that $d(p^t, Z)\xrightarrow[t \to \infty]{}0$. \hfill
\end{proof}
\vspace{0.40cm}

\begin{example} \label{exa2} Consider the following 2-player common payoff game:

\centerline{$\bordermatrix{
 & L & M & R \cr
T & (3,3)^{\r{\bullet}} & (0, 0) & (2,2) \cr
B & (0,0) & (3,3)^{\r{\bullet}}  & (2,2)
}$ }
\vspace{0.50cm}

Here we have 2 strict Nash equilibria $(T,L)$ and $(B,M)$,  and a continuum of other Nash equilibria: $(xT+(1-x)B,R)$ with $1/3\leq x \leq 2/3$.  The set $Z$ is drawn in green in the picture below, which represents the set of mixed strategies $\Delta$.  
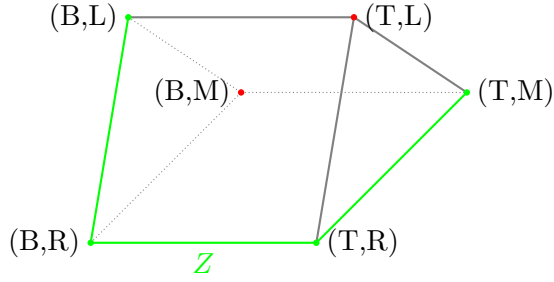
\vspace{0.50cm} 

\begin{center}
  \begin{tikzpicture}[scale=0.5]
\draw[gray, densely dotted] (0,0) -- (4,4);
\draw[green, thick] (6,0) -- (10,4);
\draw[green, thick] (0,0) -- (6,0);
 \draw[green, thick](3,0)node[below]{$Z$};
\draw[gray, densely dotted] (4,4) -- (10,4);
\draw[green, thick] (0,0) -- (1,6);
\draw[gray, densely dotted] (1,6) -- (4,4);
\draw[gray, thick] (6,0) -- (7,6);
\draw[gray, thick] (7,6) -- (10,4);
\draw[gray, thick] (1,6) -- (7,6);
\filldraw[green] (0,0) circle (2pt) node[anchor=east, black]{(B,R)};
\filldraw[red] (4,4) circle (2pt) node[anchor=east, black]{(B,M)};
\filldraw[green] (1,6) circle (2pt) node[anchor=east, black]{(B,L)};
\filldraw[green] (6,0) circle (2pt) node[anchor=west, black]{(T,R)};
\filldraw[green] (10,4) circle (2pt) node[anchor=west, black]{(T,M)};
\filldraw[red] (7,6) circle (2pt) node[anchor=west, black]{(T,L)};
\end{tikzpicture}  
\end{center}
Here almost surely the play converges to $(T,L)$, or it converges to $(B,M)$, or the distance from $p^t$ to the green set $Z$ goes to 0 when $t\to \infty$. \hfill $\Box$
 \end{example}
\vspace{0.35cm}

Let us now move on to  another consequence  of Theorem \ref{thm1}.  

\begin{definition}For each $t\geq  0$, define the random variable:
    $$L_t=\P_{p^0}(a^{t}\in \SNE \ |\ {\cal F}_{t})=p^t(\SNE).$$
      \end{definition}  $L_t$ is a random variable with values in $[0,1]$, representing the probability to play a strict Nash equilibrium at stage $t$ given the past history. Notice that if  there is no $SNE$   then $L_t=0$ for every $t$.

    \begin{corollary} \label{cor2} For any $p^0\in \Delta$, $(L_t)_t$ converges $\mathbb{P}_{p^0}$-almost surely to $0$ or $1$.
\end{corollary}

\begin{proof}
By corollary \ref{cor1}:
\begin{itemize}
    \item[-]either there exists $a\in \SNE$ such that $a^t=a$ for all sufficiently large $t$. Then by Lemma \ref{lem1},  $p^t\xrightarrow[t\to \infty]{} a$, so $L_t\xrightarrow[t\to \infty]{} 1$.
    \item[-]or $d(p^t, Z)\xrightarrow[t \to \infty]{}0$. This implies $p^t(\SNE)=L_t\xrightarrow[t\to \infty]{} 0$. \hfill
\end{itemize}
\end{proof}

\vspace{0.35cm}

\begin{example} \label{exa3} Here is an example where both events $(L_t\xrightarrow[t\to \infty]{} 1)$ and $(L_t\xrightarrow[t\to \infty]{} 0)$ have positive probability.  
\[
\bordermatrix{
 & B_1 & B_2 & B_3 \cr
A_1 & (0,0) & (-1, -1) & (-1,-1)\cr
A_2 & (-1,-1) & (1,1) & (1,0) \cr
A_3 & (-1,-1) & (1,0) & (1,1)
}.
\]
\vspace{0.3cm}

\noindent Let $p^0$ be in $\Int(\Delta)$. $(A_1,B_1)$ is {the unique} strict Nash equilibrium, so by Lemma \ref{lem1} there is a positive probability to play $(A_1,B_1)$ at each stage, so that  $(L_t\xrightarrow[t\to \infty]{} 1)$.

Define now the event, for $t\geq 0$: 
\[E_t=\cap_{s=0}^t \{a^s_1\in\{A_2,A_3\}, a^s_2\in \{B_2,B_3\}\}.\]
$E_t$ is  the event where  both $A_1$ and $B_1$ are not   played in the first $t+1$ stages. If   $E_{\infty}$ holds,  then $p_1^t(A_1)$ and $p_2^t(B_1)$ go to zero. So  it is sufficient to prove    $\P_{p^0}(E_{\infty})>0$ to obtain  a positive probability that $(L_t\xrightarrow[t\to \infty]{} 0)$. 

We have $ \P_{p^0}(E_{\infty})=\P_{p^0}(E_0)\prod_{t=0}^{\infty}(1-\P_{p^0}(E^c_{t+1}|E_t)).$ Now

\begin{align}
\P_{p^0}(a_1^{t+1}=A_1| E_t)&=\frac{p_1^0(A_1)e^{-\eta_1(t+1)}}{p_1^0(A_1)e^{-\eta_1(t+1)}+p_1^0(A_2)e^{\eta_1(t+1)}+p_1^0(A_3)e^{\eta_1(t+1)}}\nonumber\\
&=\frac{p_1^0(A_1)}{p_1^0(A_1)+(1-p_1^0(A_1))e^{2\eta_1(t+1)}},\nonumber
\end{align}
Similarly, 
$\P_{p^0}(a_2^{t+1}=B_1| E_t)=\frac{p_2^0(B_1)}{p_2^0(B_1)+(1-p_2^0(B_1))e^{2\eta_2(t+1)}}$.
Defining 
$$ c=\min\left\{\frac{1-p_1^0(A_1)}{p_1^0(A_1)}, \frac{1-p_2^0(B_1)}{p_2^0(B_1)}\right\}>0 \text{ and } \eta=\min\{\eta_1,\eta_2\}>0,$$
we obtain:
$\P_{p^0}(E_{t+1}^c|E_t)\leq \frac{2}{1+ce^{2\eta(t+1)}}$.
Hence
$$\P_{p^0}(E_{\infty})\geq \P_{p^0}(E_0) \prod_{t=1}^{\infty}\left(1- \frac{2}{1+ce^{2\eta(t+1)}}\right)>0,$$
and $(L_t\xrightarrow[t\to \infty]{} 0)$ also has positive probability. \ensuremath{\square}

\end{example}

\section{Nash Equilibria with Equalizing Payoffs}
\label{sec:neep} 

In this section we introduce the notion of Nash Equilibrium with Equalizing Payoffs (NEEP)  and prove that any almost sure limit of the EW process is a NEEP.

We first provide some intuition. Assume that with positive probability,  $(p^t)$ converges to a mixed action profile $p=(p_i)_{i \in N}$. Assume moreover that there exist  a player $i$, pure actions $a_i$, $b_i$ in the support of $p_i$, and for each player $j\neq i$ an action $a_j$ in the support of $p_j$ such that player $i$ strictly prefers  $a_i$ to $b_i$ when the other players are playing  $a_{-i}=(a_j)_{j \neq i}$. For arbitrary large $T$ there will almost surely exist $T$ consecutive stages where   $a_{-i}$ will be played. Since player $i$ strictly prefers  $a_i$ to  $b_i$, the ratio $\frac{p^t(b_i)}{p^t(a_i)}$ will become after these $T$ stages very close to 0, a contradiction with $p^t\to p$. Notice that the argument deeply rely on the fact that players react here to realized payoffs, and would not hold for the   deterministic expected-payoff version of  EW.

We now define NEEP. It  is well known that in   a Nash equilibrium $p$, if for some player $i$  two  pure actions $a_i$, $b_i$ in $A_i$ are played with positive  probability under   $p_i$, then the expected payoffs $u_i(a_i,p_{-i})$ and $u_i(b_i, p_{-i})$ coincide. In the following definition, we require this equality between payoffs to hold not only in expectation, but also almost surely. 

\begin{definition}   A  Nash Equilibrium with Equalizing Payoffs, or NEEP for short, is a Nash equilibrium $p$ in $\Delta$ such that  for each player $i$, for all $a_i$, $b_i$ in the support of $p_i$, for all $a_{-i}$ in the support of $p_{-i}$:
 $$u_i(a_i,a_{-i})=u_i(b_i,a_{-i}).$$
 We denote by ${\cal P}$ the set of NEEP.
\end{definition}
\vspace{0.25cm}

Notice that any pure Nash equilibrium is a  NEEP, and any NEEP is a Nash equilibrium.  A NEEP is a Nash equilibrium $p$ such that if $p^0=p$, then almost surely $p^t=p$ for all $t$.
Note also that if a NEEP belongs to $\Int(\Delta)$, then the game must be such that the payoff of each player depends only on the actions of the other players, and thus the Markov chain induced by the EW algorithm is almost surely constant. Therefore, in  the interesting cases  the NEEP will be located in the boundary of $\Delta$.

Existence of a NEEP is not guaranteed, for instance in the Matching Pennies game  ${(1,-1) \; (-1, 1)} \choose {(-1,1) \; (1, -1)}$, there is no NEEP. Notice that in the Chicken game ${(0,0) \; (7, 2)} \choose {(2,7) \; (6, 6)}$, the NEEP are the 2 pure Nash equilibria whereas the unique Evolutionary Stable Strategy corresponds to the fully mixed Nash.

The notion of NEEP is close to the notion of ``Weakly Stable Nash Equilibrium'' (WSNE) introduced in \cite{WSNE} (see also \cite{PP2016}) in the context of congestion games: a WSNE is  a mixed Nash equilibrium $p$  such that  for all distinct players $i$, $j$,  for all $a_i$, $b_i$ in the support of $p_i$, for all $a_{j}$ in the support of $p_{j}$: $u_i(a_i,a_j, (p_{k})_{k \neq i,j}) =u_i(b_i, a_j, (p_{k})_{k \neq i,j})$. Every NEEP is a WSNE, and for games with 2 players the two notions coincide. However the following example shows that even for potential games,  a WSNE need not always be a NEEP.   Consider the extension of Example \ref{exa1}  to $n \geq 3$ players, where each player $i$ has the pure action set $A_i=\{0,1\}$ and the (common) payoff of a player is 1  if the sum of the actions played is even, and is  0 if this sum  is odd. The mixed action profile   where each player uniformly  randomizes between his 2 actions is not a  NEEP, since when a player knows the pure actions played by all others,  he is never indifferent between action  0 and action 1. However it is a WSNE, since the expected payoff of a player is 1/2 when at least one of the other players uniformly randomizes between 0 and 1.

We now present the main result of this section.

\begin{theorem} \label{thm3} For  $p^0$ in $\Int(\Delta)$, any possible limit of the stochastic process $(p^t)_t$ is a Nash Equilibrium with Equalizing Payoffs.  Formally, 

  $$\P_{p^0}( \exists p\in \Delta\backslash {\cal P}, p^t \xrightarrow[t \to \infty]{}p)=0.$$
\end{theorem}

\vspace{0.5cm}

\begin{proof}
 Let $\Omega=\{\omega=(p^0,a^0,p^1,a^1,...,)\}$ be the set of sequences in $\Delta\times A$. The initial condition $p^0$ induces a probability $\P_{p^0}$ on $\Omega$, endowed with the $\sigma$-algebra ${\cal F}_\infty$, and we have to prove that the probability that $p^t(\omega)$ converges to a limit $p(\omega)$ in $\Delta\backslash {\cal P}$ is 0. 

By Lemma \ref{LevyExt}, there exists $\tilde \Omega \subset \Omega$ such that $\P_{p^0}(\tilde \Omega)=1$ and for each $B\subset A$, \eqref{eqLevyExt} holds.
Consider $\omega$ in $\tilde\Omega$ such that $(p^t(\omega))_t=(p^t)_t$ converges to a limit $p=p(\omega)$ in $\Delta$. Fix a   player $i$ in $N$, $a_i$ and $b_i$ in $A_i$  such that $p_i(a_i)>0$.  
\begin{align*}
\frac{  p^t_{i}(b_i)}{  p^t_i(a_i)}&=\frac{  p^0_{i}(b_i)}{  p^0_{i}(a_i)}\exp\left({\eta_i\sum_{k=0}^{t-1}( u_i(b_i,a_{-i}^k)- u_i(a_i,a_{-i}^k))}\right), \\
& = \frac{  p^0_{i}(b_i)}{  p^0_{i}(a_i)}\exp\left({t \eta_i\frac{1}{t}\sum_{k=0}^{t-1}( u_i(b_i,a_{-i}^k)- u_i(a_i,a_{-i}^k))}\right).\end{align*}
 The sequence $(\frac{  p^t_{i}(b_i)}{  p^t_i(a_i)})_t$ is bounded, so \begin{equation}\label{eq5}\limsup_{t \to \infty} \frac{1}{t}\sum_{k=0}^{t-1}( u_i(b_i,a_{-i}^k)- u_i(a_i,a_{-i}^k))\leq 0.\end{equation}
 Consider any $b_{-i}\in A_{-i}$. By Lemma \ref{LevyExt}, $\frac{1}{t}\sum_{k=0}^{t-1} \1_{a_{-i}^k=b_{-i}}\xrightarrow[t\to \infty]{}p_{-i}(b_{-i})$ and therefore 
\begin{align*}
\lim_{t\to\infty}\frac{1}{t}\sum_{k=0}^{t-1} u_i(b_i,a_{-i}^k)&=\lim_{t\to\infty}  \frac{1}{t}\sum_{k=0}^{t-1}\sum_{b_{-i}\in A_{-i}}  \1_{a_{-i}^k=b_{-i}}u_i(b_i,b_{-i}) \\
&=\sum_{b_{-i}\in A_{-i}} u_i(b_i,p_{-i}(b_{-i}))\\
& = u_i(b_i,p_{-i}).
\end{align*} 
By inequality (\ref{eq5}), we get $u_i(b_i,p_{-i})\leq u_i(a_i,p_{-i})$ for all $b_i, a_i$ in $A_i$ such that $p_i(a_i)>0$, that is $p$ is a Nash equilibrium.

Fix now  $i\in  N$ and consider any  $a_i$ such that $p_i(a_i)>0$. 
We have for all $t$:
\[p_{i}^{t+1}(a_i)=  \frac{p^t_{i}(a_i)e^{\eta_i u_i(a_i,a_{-i}^t)}}{\sum\limits_{b_i\in A_i}  p^t_{i}(b_i)e^{\eta_i u_i(b_i,a_{-i}^t)}}.\]
Since  $p_{i}^{t+1}(a_i)$ and $p^t_{i}(a_i)$ converge to the same positive limit $p_i(a_i)$, it follows that
\[\lim_{t\to\infty}\frac{e^{\eta_i u_i(a_i,a_{-i}^{t})}}{\sum\limits_{b_i\in A_i}  p^t_{i}(b_i)e^{\eta_i u_i(b_i,a_{-i}^{t})}}= 1. \]
Now, fix any  $a_{-i}\in A_{-i}$ such that $p_{-i}(a_{-i})>0$. Recall that by Lemma \ref{LevyExt},   the set  $\{t:a_{-i}^t=a_{-i}\}$ is almost surely  infinite. Therefore, by considering an appropriate  subsequence (which depends on $\omega$), we have
\[ \frac{e^{\eta_i u_i(a_i,a_{-i})}}{\sum\limits_{b_i\in A_i}p_{i}(b_i)e^{\eta_i u_i(b_i,a_{-i})}}= 1.\]
We conclude that the quantity 
\[ e^{\eta_i u_i(a_i,a_{-i})}=\sum_{b_i\in A_i}p_{i}(b_i)e^{\eta_i u_i(b_i,a_{-i})}\]
does not depend on the choice of $a_i$ such that $p_i(a_i)>0$, and therefore that $p$ is a NEEP. \hfill

\end{proof}

\vspace{0.5cm}

Note that for games where payoffs are all different ($u_i(a)\neq u_i(b)$ for all $i$ in $N$, $a\neq b$ in $A$), the definition of NEEP reduces to that of strict Nash equilibrium. So the following corollary of Theorem \ref{thm3}, where genericity refers to the genericity of the vector payoff in $\R^{A\times N}$ with respect to the Lebesgue measure,  is immediate.

\begin{corollary} For generic games, the EWA process $(p^t)_t$ with $p^0\in \Int(\Delta)$ can only converge to a strict  Nash equilibrium.
\end{corollary}
\vspace{0.5cm}
 
\begin{remark}\label{remark_ftrl}
In \cite{giannou2021} (Theorem 2),  for a class of dynamics essentially  including ours as a particular case (allowing for constant learning rates but with the same rates for every player), it is proven that in a generic game, if $p^*$ is a non-pure Nash equilibrium, then $p^*$ is not stochastically asymptotically stable for the dynamics, in the sense given in \cite{giannou2021}. It is easy to check that if $p^*$ is stochastically asymptotically stable, then the following  property holds 
\begin{equation} \label{property_SAS}
\exists p_0 \in int(\Delta), \; \P_{p_0}(p^t \xrightarrow[t \to \infty]{}p^*)>0.
\end{equation}
Using our terminology, the genericity assumption is essentially used in \cite{giannou2021} to guarantee that $p^*$ is not a NEEP (see Lemma B.1 in \cite{giannou2021}). Our Theorem \ref{thm3}   implies that if $p^*$ is not a NEEP, then \eqref{property_SAS} does not hold, and thus $p^*$ is not stochastically asymptotically stable, showing that our statement  strengthens the conclusion of Theorem 2 in \cite{giannou2021} for the particular class of EW with constant learning rates.
It is also interesting to compare the results of section \ref{sec:strictNE} here with  \cite{giannou2021}. In \cite{giannou2021} (Theorem 1), it is proven that a strict Nash equilibrium is stochastically asymptotically stable for a class of dynamics which includes EW with decreasing learning rates. It is easy to check that point b) in our  Lemma \ref{lem1} implies that a strict Nash equilibrium is stochastically asymptotically stable. Thus, Lemma \ref{lem1}, together with our other results of section \ref{sec:strictNE}, provide qualitative properties for   strict Nash equilibria which are stronger than  stochastic asymptotic stability.
\end{remark}
\vspace{0.5cm}

 Theorem \ref{thm3} also implies that the EW dynamics can not converge in games without NEEP. In particular,   EW does not converge in the Matching Pennies game, as already known  from \cite{MWUstochastinZSG}. However, convergence   may fail even if the game admits a NEEP, as shown by our first example \ref{ex1111} where payoffs are given by 
  
 \[
    \bordermatrix{
     & L & R  \cr
    T & (1,1) & (1, 0) \cr
    B & (1,0) & (1,1)
    }.\]
Here the NEEP  are the elements of $E_1 \cup E_2$, with $E_1=\{(xT+(1-x)B,R), x\in [0,1/2]\}$, $E_2=\{(xT+(1-x)B,L), x\in [1/2,1]\}$.

Example \ref{ex1111} shows that convergence of the EW dynamics may fail even when NEEP exists, but only for a specific starting distribution $p^0$. The next Lemma proves that this phenomenon can be robust with respect to the initial distribution.
\begin{lemma}\label{exaNC}
    There exists a finite game admitting a NEEP such that, for every initial condition $p^0 \in int(\Delta)$, 
    $$\Prob_{p^0}\left((p^t)_t \textit{ converges}\right)=0.$$
\end{lemma}
\begin{proof}
  Consider the following game
\[
\bordermatrix{
 & L &  M & R\cr
T & (1, 0) & (0,1) & (1,0)\cr
B & (1,1) & (1,0)  & (0,1)
}.
\] 
Here the NEEP are $\{ ((1-x)T+xB, L): \ \forall x\in [1/2,1]\}$.
Fix $p^0\in int(\Delta)$ and assume, by contradiction, that $(p^t)_t$ converges with positive probability.
By Theorem \ref{thm3}, if EW dynamics converges, the limit must be a NEEP. Therefore, on this event, $p_2^t(L)\to 1$, so $p_2^t(R)\to 0$ and consequently
$$\lim_{t\to\infty}\frac{p_2^t(R)}{p_2^t(L)}=0.$$
However
$$p_2^t(L)=\frac{p^0_2(L)}{p^0_2(L)+p^0_2(R)+p^0_2(M)e^{\sum_{k=0}^{t-1} \1_{a_1^k=T}-\1_{a_1^k=B}}},$$
$$p_2^t(R)=\frac{p^0_2(R)}{p^0_2(L)+p^0_2(R)+p^0_2(M)e^{\sum_{k=0}^{t-1} \1_{a_1^k=T}-\1_{a_1^k=B}}}.$$
This implies that $p_2^t(R)/p_2^t(L)=p_2^0(R)/p_2^0(L)>0$ contradicting the fact that the same ratio must converge to $0$. 
\hfill
\end{proof}

\vspace{0.25cm}

Assume now that the  convergence of $(p^t)_t$ holds almost surely, what are the possible limits? Theorem \ref{thm3} implies that $\lim_t p^t$ has to be a NEEP, and a     NEEP is a rest point of the EW  dynamics. However the following example shows that even an isolated NEEP is not guaranteed to be the limit of a EW  process  with interior  initial condition.

\begin{example}\label{exa185} Consider the game  $$\bordermatrix{
     & L & R  \cr
    T & (1,1) & (0, 1) \cr
    B & (1,0) & (1,1)
    }.$$
Here there are 2 Nash equilibria which are pure, hence Payoff Equalizing:
${\cal P}=\{(T,L), (B,R)\}.$
However since $T$ and $L$ are   dominated strategies, every EW  process starting from an interior initial condition   will converge to $(B,R)$. $(T,L)$ is an isolated  NEEP  but can not be the limit of a EW process with interior  initial condition.
 \end{example}

In the example below, we compute the set of possible limits of EW processes with interior  initial condition. It  coincides   with the set of NEEP with maximal support (among the set ${\cal P}$), which is open and not connected.

\begin{example}\label{exa18} Consider the game  $$\bordermatrix{
     & L & R  \cr
    T & (0,0) & (0, 0) \cr
    B & (0,0) & (-1,-1)
    }.$$
Here any Nash Equilibrium is Payoff Equalizing and 
\[{\cal P}=\{(p_1,p_2), p_1(T)=1\; {\rm{or}}\; p_2(L)=1\}.\]
Among these NEEP, each element has {maximal support in ${\cal P}$ (i.e. no element of ${\cal P}$ has a strictly larger support)} except $(T,L)$, $(B,L)$ and $(T,R)$ which are pure but not strict Nash Equilibria. 

Fix $p^0\in \Int(\Delta)$, we have for each $t\geq 0$:
\[p_1^t(T)=\frac{p_1^0(T)}{p_1^0(T)+p_1^0(B) e^{-\eta_1R_t}} \; {\rm and}\; p_2^t(L)=\frac{p_2^0(L)}{p_2^0(L)+p_2^0(R) e^{-\eta_2B_t}},\]
where $R_t:=\sum_{s=0}^{t-1}\1_{a_2^s=R}$ is the number of $R$ played by Player 2 before stage $t$, and $B_t=\sum_{s=0}^{t-1}\1_{a_1^s=B}$ is the number of $B$ played by Player 1 before stage $t$. The sequences $(R_t)$ and $(B_t)$ are non decreasing, so   the sequences $(p_1^t(T))$ and $(p_2^t(L))$ are non decreasing and converge in $[0,1]$. Since $p^0\in \Int(\Delta)$, it is obviously not possible to converge to $(B,L)$ or $(T,R)$. It is less obvious and more interesting to see that the limit cannot be $(T,L)$.
\end{example}

\begin{lemma} \label{lem61} In Example \ref{exa18}, for any $p^0$ in $\Int(\Delta)$,  $(p^t)_t$ converges $\P_{p^0}$-almost surely to a NEEP {\it with maximal support}.
\end{lemma}

The proof is in the Appendix. 
From the computations therein, it is not difficult to see that in Example \ref{exa18}, any NEEP with maximal support is the limit of $(p^t)$ with positive probability for a well chosen initial condition $p^0$ in $\Int(\Delta)$ (using that the probability to always play $T$ (or $R$) is positive). Consequently, the set of possible limits of EW  processes is  open and not connected. 

Example  \ref{exa18}  raises the following open  question: is it possible to have with positive probability $p^t\xrightarrow[t\to \infty]{} p$, where $p$ is not a NEEP with maximal support ? In particular, can we have convergence to a pure Nash equilibrium which is not strict ?

\section{A Convergence Theorem for Strong Coordination Games}
\label{sec:coord}
Recall the game in Example \ref{exa1}:  
\[
    \bordermatrix{
     & L & R  \cr
    T & (1,1)^{\r{\bullet}} & (0, 0) \cr
    B & (0,0) & (1,1)^{\r{\bullet}}
    }.
\]
\vspace{0.25cm}

Here we have 2 NEEP which are strict Nash Equilibria,  and we showed  that almost surely the EW process converges to one of them.  Consider the extension to 3 actions per player, that is the game $H$ given by:
 \[  \bordermatrix{
     & l & m&r  \cr
    T & (1,1)^{\r{\bullet}} & (0, 0)& (0,0) \cr
    M & (0,0) & (1,1)^{\r{\bullet}}& (0,0) \cr
    B &(0,0) & (0,0)& (1,1)^{\r{\bullet}}
    }.
\]
An obvious conjecture is that here the play will almost  surely converge to one of the 3 strict NE. Applying corollary \ref{cor1}, we get that almost  surely either the play will  converge to one of the 3 strict NE, or $d(p^t,Z)\xrightarrow[t\to \infty]{}0$, where 
$$Z=\{p\in \Delta, p((T,l))=p((M,m))=p((B,r))=0\}.$$
$Z$ is here a connected set, union of 6 segments: 
\begin{align*}
 Z = &[(T,m),(T,r)] \cup [(T,r),(M,r)] 
 \cup [(M,r),(M,l)]   \\
 &\cup [(M,l),(B,l)]  \cup [(B,l),(B,m)]  \cup [(B,m),(T,m).
\end{align*}
To prove the conjecture we need to show that   $d(p^t,Z)\xrightarrow[t\to \infty]{}0$ cannot happen with positive probability, and this is surprisingly not so simple. Formally, it will be a consequence of the following Theorem \ref{thm4}.
 
 \begin{definition} An $n$-player finite strategic game is called a {\it strong coordination game} if all players have the same    set of pure strategies, and for each player $i$ and action profile $a$, $u_i(a)>0$ if all the players play the same action in $a$, and $u_i(a)=0$ otherwise. \end{definition}
 
In a  strong coordination game, the NEEP coincide with the strict Nash equilibria, they are the elements of the diagonal of $A$.  

\begin{theorem} {\label{thm4}
In a strong coordination game, for any $p^0$ in $\Int(\Delta)$,  the EW process converges $\P_{p^0}$-almost surely} to one of the strict Nash equilibria.
\end{theorem}

\noindent{\it Proof of Theorem \ref{thm4}.} Denote by  $S=\{1,...,m\}$ the set of pure strategies of each player $i=1,...,n$.  If $n=2$ the proof follows from the proof of Example \ref{exa1}, so in the sequel we assume $n\geq 3$.  Fix $p^0\in \Int(\Delta)$. 

Define for each $k$ in $S$, $Z_t(k)= \frac{1}{\prod_{i=1}^n p^t_i(k)}\geq 1$ and  $$Z_t=\min_{k\in S}  Z_t(k)=\frac{1}{\max_{k\in S}\; \prod_{i=1}^n p^t_i(k) }\in [1,+\infty].$$
We will show that  almost surely there exists $k$ in $S$ and $t_0$ such that for all $t\geq t_0$, $a^t=(k,...,k)$. Notice that (by Lemma \ref{lem1}) this is equivalent to showing that $Z_t\xrightarrow[t\to \infty]{}1$ almost surely.

The quantity   $\min_{k\in S} \frac{1}{\prod_{i=1}^n p_i(k)} $ will here play    the role of a potential.  The  following lemma states that $(Z_t)_t$ is a super-martingale for large values of the potential. 
  
\begin{lemma} {\label{LemmaCoordination} For $M_0$ large enough (only depending on $m$, $n$ and the payoffs of the game), 
$$\forall t\geq 0, \; \;\;\; \ind_{Z_t\geq M_0}\E(Z_{t+1} | \mathcal{F}_t )\leq Z_t\ind_{Z_t\geq M_0}.
$$}
\end{lemma}

\begin{proof}{ Fix $t\geq 0$ and assume 
without loss of generality that $Z_t=Z_t(1)$.}
 
We have 
$\E(Z_{t+1} | \mathcal{F}_t )\leq \E(Z_{t+1}(1)| \mathcal{F}_t ).$ By definition of the EW process,  for each player $i$:
$$\frac{1}{p_i^{t+1}(1)}= \frac{1}{p_i^{t}(1)} \sum_{k \in S} p_i^t(k) e^{\eta_i(u_i(k,a_{-i}^t) -u_i(1,a_{-i}^t))}.$$
\noindent so $$\E(Z_{t+1}(1)| \mathcal{F}_t )=\left. \E\left(\frac{1}{\prod\limits_{i=1}^n p^{t+1}_i(1)}\right| \mathcal{F}_t\right)= \frac{X_t}{\prod\limits_{i=1}^n p^{t}_i(1)},$$
with $$X_t=\sum_{\hat{k}=(k_1,...,k_n)\in S^n} X_t(\hat{k})\prod_{i=1}^n p_i^t(k_i),$$
and for all $\hat{k}$,  $$X_t(\hat{k})= \prod_{i \in N} \sum_{k\in S} p_i^t(k)e^{\eta_i(u_i(k,\hat{k}_{-i}) -u_i(1,\hat{k}_{-i}))}.$$
For every $k\in S$ denote $\alpha_i(k)=e^{\eta_i u_i(k,...,k)}-1>0$. 

We first compute $X_t(\hat{k})$ for each $\hat{k}$ in $S^n$. 
\begin{itemize}
\item if $\hat{k}=(1,...,1)$, $X_t(\hat{k})=\prod_{i \in N}\left(p_i^t(1)+(1-p_i^t(1)) e^{-\eta_iu_i(1,...,1)}\right) =\prod_{i \in N} \frac{1+p_i^t(1)\alpha_i(1)}{1+\alpha_i(1)}$,
\item if $\hat{k}=(k,...,k)$ with $k\neq 1$, 
\[X_t(\hat{k})=\prod_{i \in N}\left((1-p_i^t(k))+p_i^t(k) e^{\eta_i u_i(k,...,k)}\right)=\prod_{i \in N} (1+p_i^t(k)\alpha_i(k)),\]
\item if $\hat{k}= (1_{-i},k)$ for some player $i$ and action $k\neq 1$, $X_t(\hat{k})= \frac{1+p_i^t(1)\alpha_i(1)}{1+\alpha_i(1)}$,
\item if $\hat{k}=(k_{-i},k')$ for some player $i$, $k'\in S$ and    $k\neq 1$, $k\neq k'$,   $X_t(\hat{k})= 1+p_i^t(k) \alpha_i(k))$,
\item otherwise, in $\hat{k}$ there is no action played by all or by $n-1$ players, so that for every player $i$ the set $\{k_j, j\neq i\}$ has at least two elements, and then $X_t(\hat{k})=1$.
\end{itemize}

Denote by $\gamma$ the probability under $p^t$ that  the action profile played is such that there is no action played simultaneously by all or by $n-1$ players. We obtain: 
\begin{align*}
X_t =\gamma &+ \left( \prod_{i\in N} p_i^t(1)\right) \left(  \prod_{i \in N} \frac{(1+p_i^t(1)\alpha_i(1))}{1+\alpha_i(1)}\right)   + \sum_{k=2}^m \left( \prod_{i\in N} p_i^t(k)\right) \left(  \prod_{i \in N} (1+p_i^t(k)\alpha_i(k))\right) \\
  &+ \sum_{k=2}^m \sum_{i \in N} \left(\prod_{j\neq i} p_j^t(1)\right) p_i^t(k)  \frac{(1+p_i^t(1)\alpha_i(1))}{1+\alpha_i(1)}  \\
  &+  \sum_{k=2}^m \sum_{i \in N} \left(\prod_{j\neq i} p_j^t(k)\right)  (1-p_i^t(k))(1+p_i^t(k) \alpha_i(k)).
\end{align*}
Using that
$$1=\gamma  +  \left( \prod_{i\in N} p_i^t(1)\right) + \sum_{k=2}^m \left( \prod_{i\in N} p_i^t(k)\right) + \sum_{k=2}^m \sum_{i \in N} \left(\prod_{j\neq i} p_j^t(1)\right)p_i^t(k)  
 + \sum_{k=2}^m \sum_{i \in N} \left(\prod_{j\neq i} p_j^t(k)\right)(1-p_i^t(k)),$$
we get:
\begin{align*}
X_t -1 =  & \left( \prod_{i\in N} p_i^t(1)\right) \left( -1+ \prod_{i \in N} \frac{(1+p_i^t(1)\alpha_i(1))}{1+\alpha_i(1)}\right)  \\
& + \sum_{k=2}^m \left( \prod_{i\in N} p_i^t(k)\right) \left( -1+ \prod_{i \in N} (1+p_i^t(k)\alpha_i(k))\right) \\
 &+  \sum_{k=2}^m \sum_{i \in N} \left(\prod_{j\neq i} p_j^t(1)\right) p_i^t(k)  \frac{\alpha_i(1)( p_i^t(1)-1)}{1+\alpha_i(1)}   \\
&+  \sum_{k=2}^m \sum_{i \in N} \left(\prod_{j\neq i} p_j^t(k)\right)  (1-p_i^t(k))p_i^t(k) \alpha_i(k)\\
=& \left( \prod_{i\in N} p_i^t(1)\right) \left( -1+ \prod_{i \in N} \frac{(1+p_i^t(1)\alpha_i(1))}{1+\alpha_i(1)}\right)  \\
& + \sum_{k=2}^m \left( \prod_{i\in N} p_i^t(k)\right) \left( -1+ \prod_{i \in N} (1+p_i^t(k)\alpha_i(k))\right) \\
 &+  \sum_{k=2}^m \sum_{i \in N} \left(\prod_{j \in N} p_j^t(1)\right) p_i^t(k)  \frac{\alpha_i(1)}{1+\alpha_i(1)}-\sum_{k=2}^m \sum_{i \in N} \left(\prod_{j\neq i} p_j^t(1)\right) p_i^t(k)  \frac{\alpha_i(1)}{1+\alpha_i(1)}   \\
&+  \sum_{k=2}^m \sum_{i \in N} \left(\prod_{j\in N} p_j^t(k)\right)  (1-p_i^t(k))\alpha_i(k).
\end{align*}
We know that $\prod_{i\in N} p_i^t(1)= \frac{1}{Z_t}$ and $\prod_{i\in N} p_i^t(k)\leq \frac{1}{Z_t}$ for each $k$. So using that
\[ -1+ \prod_{i \in N} (1+p_i^t(k)\alpha_i(k)) \geq 0,\]
we get:
\begin{align*}
X_t -1 \leq   \frac{1}{Z_t}    \Bigg\{ &-1 + \prod_{i \in N} \frac{(1+p_i^t(1)\alpha_i(1))}{1+\alpha_i(1)}   \\
 &-(m-1)  + \sum_{k=2}^m    \prod_{i \in N} (1+p_i^t(k)\alpha_i(k))   \\
 & +\sum_{k=2}^m \sum_{i \in N}  p_i^t(k)  \frac{\alpha_i(1)}{1+\alpha_i(1)} \\
 & + \sum_{k=2}^m \sum_{i \in N}    (1-p_i^t(k))  \alpha_i(k) \Bigg\} -  \sum_{k=2}^m \sum_{i \in N} \left(\prod_{j\neq i} p_j^t(1)\right) p_i^t(k)  \frac{\alpha_i(1)}{1+\alpha_i(1)}.
\end{align*}
The last term of the above expression can be written as:
\begin{align*} 
-\sum_{k=2}^m \sum_{i \in N} \left(\prod_{j\neq i} p_j^t(1)\right) p_i^t(k)  \frac{\alpha_i(1)}{1+\alpha_i(1)} &= - \sum_{i \in N} \left(\prod_{j\neq i} p_j^t(1)\right) (1-p_i^t(1))  \frac{\alpha_i(1)}{1+\alpha_i(1)} \\
&= \frac{1}{Z_t} \Bigg\{ \sum_{i \in N}  \frac{\alpha_i(1)}{1+\alpha_i(1)}\Bigg\}- \sum_{i \in N} \left(\prod_{j\neq i} p_j^t(1)\right)  \frac{\alpha_i(1)}{1+\alpha_i(1)} .
\end{align*}
Finally, 
\[X_t -1 \leq \frac{C}{Z_t}-D\sum_{i=1}^n\prod_{j\neq i} p^t_j(1),\]
with 
\[C= -m+ 1+\sum_{k=2}^m\prod_{i\in N}(1+\alpha_i(k))+\sum_{i=1}^n  \frac{\alpha_i(1)}{1+\alpha_i(1)}+\sum_{k=2}^m \sum_{i \in N}  \alpha_i(k)+ \sum_{i=1}^n  \frac{\alpha_i(1)}{1+\alpha_i(1)}>0\] 
\[D=\min_{i\in N} \frac{\alpha_i(1)}{1+\alpha_i(1)}>0.\]
Now, using the arithmetic and geometric means  inequality to the numbers $(\prod_{j \neq i} p_j^t(1))_{i \in N}$ we obtain:
\[\frac1n \sum_{i=1}^n\prod_{j\neq i} p^t_j(1)\geq   \left(\prod_{j=1}^n p^t_j(1)\right)^{\frac{n-1}{n}}= \left(\frac{1}{Z_t}\right)^{\frac{n-1}{n}}.\]
Thus
\[X_t-1\leq \frac{C}{Z_t}-D n \left(\frac{1}{Z_t}\right)^{\frac{n-1}{n}}, \;{\rm and}\; \; Z_t(X_t-1)\leq C-Dn{Z_t}^{1/n}.\]
Hence, for  $M_0$ large enough we have that  $(Z_t>M_0)$ implies  $(X_t\leq 1)$, and this  concludes the proof  of  Lemma \ref{LemmaCoordination}. \hfill 
\end{proof}

We  now finish the proof of Theorem \ref{thm4}. By Lemma \ref{LemmaCoordination}, there exists $M_0$ that depends only on the structure of the game such that 
 \begin{equation} \label{Martingale}
\ind_{Z_t\geq M_0} \E[Z_{t+1} | \mathcal{F}_t ]\leq Z_t\ind_{Z_t\geq M_0}.
 \end{equation}
Define  $\tau=\inf\{t\geq 0,  Z_t<M_0\}$ and $\hat{Z}_t=Z_{t\wedge \tau}$. The event $(t<\tau)$ is  $ \mathcal{F}_t$-measurable and implies   $\hat{Z}_t=Z_t\geq M_0$ and $\hat{Z}_{t+1}=Z_{t+1}$; so using \eqref{Martingale}
$$\E(\1_{t<\tau} \hat{Z}_{t+1}| \mathcal{F}_t )= \E(\1_{t<\tau}{Z}_{t+1}| \mathcal{F}_t ) =\1_{t<\tau}\E( {Z}_{t+1}| \mathcal{F}_t )\leq \1_{t<\tau}Z_t=\1_{t<\tau}\hat{Z}_t.$$
And if $t\geq \tau$ then $\hat{Z}_t=Z_{\tau}=\hat{Z}_{t+1}$, so
$\E(1_{t\geq \tau} \hat{Z}_{t+1}| \mathcal{F}_t)=\E(1_{t\geq \tau}\hat{Z}_{t}| \mathcal{F}_t )=1_{t\geq \tau}\hat{Z}_{t}.$
We obtain that $(\hat{Z}_{t})_{t\geq 0}$ is a  supermartingale with values in $\R_+$.  So it converges almost surely to a limit $\hat{Z}_{\infty}<\infty$.\\

If $\tau=\infty$ then $\hat{Z}_{t}=Z_t$ converges almost surely to $\hat{Z}_{\infty}<\infty$, i.e. $Z_t$ doesn't converge to $+\infty$. Hence
$\Prob\left(\tau=\infty, \lim\limits_{t\to\infty}Z_t=\infty\right)=0,$
or equivalently:
$$\Prob\left(\forall t \ Z_t\geq M_0 \mbox{ and } \lim\limits_{t\to\infty}Z_t=\infty \right)=0.$$
\noindent Thus,  
\begin{align*}\Prob_{p^0}\left(\lim\limits_{t\to\infty}Z_t =\infty\right)&= 
\Prob\left(\exists t^{*} \ s.t. \ \forall t\geq t^{*}  \ Z_t\geq M_0 \ \mbox{ and } \lim\limits_{t\to\infty}Z_t=\infty \right)\\
&\leq \sum_{t^{*}=0}^{\infty}\Prob\left(  \forall t\geq t^{*} \ Z_t\geq M_0 \ \mbox{ and } \lim\limits_{t\to\infty}Z_t=\infty \right)=0,
\end{align*}
where the last equality is obtained using the Markov property.

Notice that if $p^t(SNE)\xrightarrow[t\to \infty]{}0$, that is if the probability of playing  a strict NE at stage $t$ converges to 0, then $Z_t \xrightarrow[t\to \infty]{}\infty$. Since  $\Prob_{p^0}\left(\lim\limits_{t\to\infty}Z_t =\infty\right)=0$, we get $\Prob_{p^0}\left(p^t(SNE)\xrightarrow[t\to \infty]{}0 \right)=0$. So almost surely there exists $k$ such that $(p^t(k,...,k))_t$ does not converge to 0, and by Theorem \ref{thm1}, there exists $t_0$ such that for all $t\geq t_0$, $a^t=(k,...,k)$. This concludes the proof of Theorem \ref{thm4}. \hfill $\Box$

\vspace{0.5cm}

We conclude this section by discussing why the argument above does not directly extend to broader classes such as common-payoff or potential games. The proof of Theorem \ref{thm4} relies on the specific super-martingale constructed in Lemma \ref{LemmaCoordination}, which uses the diagonal/off-diagonal structure of strong coordination games. In the next example, we show that even a slight modification of a strong coordination game may require the construction of a different super-martingale in order to obtain last-iterate convergence.

\begin{example} \label{exa7}
Consider   the 2-player game:
\[    
\bordermatrix{
 & l & m &  r \cr
T & (1,1)^{\r{\bullet}} & (0, 0) & (0,0) \cr
M & (0,0) & (1,1)^{\r{\bullet}}  & (0,0) \cr
B & (0,0) & (\frac{1}{2},\frac{1}{2})  & (1,1)^{\r{\bullet}}
}
\]
This is a close variant of the strong coordination game  
\[  \bordermatrix{
     & l & m&r  \cr
    T & (1,1)^{\r{\bullet}} & (0, 0)& (0,0) \cr
    M & (0,0) & (1,1)^{\r{\bullet}}& (0,0) \cr
    B &(0,0) & (0,0)& (1,1)^{\r{\bullet}}
    },\]
hence one may try the same approach by considering 
\[Z_t=\min_{k\in S}\frac{1}{\prod_{i=1}^n p^t_i(k)}\in [1,+\infty].\]
However here the analog of Lemma \ref{LemmaCoordination} does not hold: one can show that for every $M_0>0$ there exists $p_0 \in$  $\Int(\Delta)$ such that $Z_0\geq M_0$ and $\E_{p^0}(Z_1)>Z_0$. 
    
For this game we can prove almost sure convergence to a strict NE by introducing 
    
    $$Z'_t=\min\left\{\frac{1}{x^t_1y_1^t}, \frac{1}{y_2^t(x_2^t+x_3^t)}, \frac{1}{x_3^t(y_2^t+y_3^t)}\right\}.$$
\vspace{0.1cm}

\noindent Hence, as in the proof of Theorem \ref{thm4} we can show that  if $M_0$ is large enough, for all $t>0$ we have:

$$\ind_{Z'_t\geq M_0}\E[Z'_{t+1}|\mathcal{F}_t]\leq Z'_t\ind_{Z'_t\geq M_0};$$
and

    $$\P_{p^0}\left(\lim_{t\to\infty} Z'_t=\infty\right)=0.$$
    
Notice that if $p^t(SNE)\xrightarrow[t\to \infty]{}0$, since  $\Prob_{p^0}\left(\lim\limits_{t\to\infty}Z'_t =\infty\right)=0$, we get that $p^t=(x^t,y^t)$ converges to the set

$$D:=\{(x,y)\in\Delta: x_3=1, y_3=0\}\cup\{(x,y)\in\Delta: y_2=1, x_2=0\}.$$

However, it implies that $y^t_3\xrightarrow[t\to \infty]{}0$, then by Lemma \ref{LevyExt} it simple to show that $x^t_3\xrightarrow[t\to \infty]{}0$. Hence, the unique option is that $p^t$  converges to the pure strategy $(T,m)$ but it is in contradiction with Theorem \ref{thm3}, so $\Prob_{p^0}\left(p^t(SNE)\xrightarrow[t\to \infty]{}0 \right)=0$. It concludes the proof and shows that also in the Example \ref{exa7} the dynamics induced by the EW converge a.s. to a SNE.\ensuremath{\square} 
\end{example}

\vspace{0.25cm}

\begin{remark}\label{remark_potential} In order to try to mimic the proof of Theorem \ref{thm4} for common payoff games, we may  consider natural potential functions of these games. However, as we show below, the two most natural candidates$-$
1) the probability to play a strict NE, and 2) the common expected payoff$-$ do not work.
Consider again  the game of Example   \ref{exa1}:  
\[
\bordermatrix{
     & L & R  \cr
    T & (1,1)^{\r{\bullet}} & (0, 0) \cr
    B & (0,0) & (1,1)^{\r{\bullet}}
    }.
\]
Here the two candidates coincide, and we denote by 
\[u(p)=xy+(1-x)(1-y)\]

\noindent the common expected payoff, or the probability to play a strict NE,  when
\[ p=(xT+(1-x)B,yL+(1-y)R)\]
is played.  It would be nice if $(u(p^t))_t$ was a sub-martingale, at least\footnote{It is easy to  construct  examples where the current expected payoff is not a submartingale for large $\eta_1$ and  $\eta_2$.}for small $\eta_1$ and $\eta_2$. However the following computations show that it is not the case.

For $x$, $y$ in $[0,1]$,  define
\[f_{a,b}(x,y)=\E_{p^0}(u(p^1)) -u(p^0)\]
for $p^0=(xT+(1-x)B,yL+(1-y)R)$. Computations show that
\begin{align*}
  f_{a,b}(x,y) D = & -u(1+ax)(1+by)(1+a(1-x))(1+b(1-y))\\
   &+ xy (u+xy(a+b+ab))(1+a(1-x))(1+b(1-y)) \\
   & + x(1-y)(u+a(1-x)(1-y)+bxy)(1+ax)(1+b(1-y))\\
   &+(1-x)y(u+axy+b(1-x)(1-y))(1+by)(1+a(1-x)) \\
   &+ (1-x)(1-y)(u+(1-x)(1-y)(a+b+ab))(1+ax)(1+by),
\end{align*}
with $u=u(x,y)$,   $a=e^{\eta_1}-1>0$, $b=e^{\eta_2}-1>0$, and 
\[D=(1+ax)(1+by)(1+a(1-x))(1+b(1-y))>0.\] 
Notice that $f_{0,0}(x,y)=0$ for all $(x,y)$, and $f_{a,b}(1/2,1/2)=0$ for all $(a,b)$. If both $x$ and $y$ are not in $\{0,1/2,1\}$  one can show that there exists $\alpha>0$ such that  for $a$, $b$ $\in (0,\alpha)$ we have $f_{a,b}(x,y)>0$.
   
   However take $a$ and $b$ very close to 0, with $a<<b$, and define $x=\frac{1}{2} -\frac{b}{16}$ and $y=1/4$. One can check that in this case,   $f(x,y)  \sim -\frac{3b^3}{2048}<0$. 

\end{remark}

\section{Simulations}
\label{sec:simulations}
 
Consider once again the game of Example   \ref{exa1}:  
\[ \bordermatrix{
     & L & R  \cr
    T & (1,1)^{\r{\bullet}} & (0, 0) \cr
    B & (0,0) & (1,1)^{\r{\bullet}}
    }.\] 
We know that for any initial probability $p^0\in \Int(\Delta)$, almost surely the play converges to $(T,L)$ or to $(B,R)$. But how does the probability of reaching $(T,L)$ depend on $p^0$ ?
\vspace{0.35cm}

Let us define for any $(x,y)\in [0,1]^2$,   the probability $f(x,y)$  to converge to $(T,L)$ starting from  $p^0=(xT+(1-x)B,yL+(1-y)R).$  Writing   $a=e^{\eta_1}-1$ and  $b=e^{\eta_2}-1$, $f$ can be characterized as  the unique  function from $[0,1]^2$ to $\R$ satisfying the following  three conditions:
\begin{enumerate}
\item $ \forall (x,y)\in [0,1]^2,$
\begin{align*}
f(x,y)& =
  xy f\left(\frac{x(1+a)}{1+ax}, \frac{y(1+b)}{1+by}\right) + x(1-y)  f\left(\frac{x}{1+a(1-x)}, \frac{y(1+b)}{1+by}\right)\\
  & +   (1-x)y  f\left(\frac{x(1+a)}{1+ax}, \frac{y}{1+b(1-y)}\right) \\
 &+(1-x)(1-y) f\left(\frac{x }{1+a(1-x)},\frac{y}{1+b(1-y)}\right),
 \end{align*}
\item $f$ is continuous on $[0,1]^2\backslash\{(1,0),(0,1)\}$, 
\item $f$ satisfies the boundary conditions $f(0,0)=f(0,1)=f(1,0)=0$   and $f(1,1)=1$.
\end{enumerate}

\vspace{0.25cm}

However  obtaining an explicit formula for $f$ is probably out of scope, as there are classical examples in dimension one quite similar to our problem for which no explicit formula is known \cite{Karlin}.

If $a=b$, we have by symmetry $f(x,y)=f(y,x)=1-f((1-x),(1-y))$, so that $f(x,1-x)=1/2$ for all $x$ in $(0,1)$.   In order to  have a glimpse of the graph of $f$ , we have considered a grid 11x11 of initial distributions starting from zero with mesh 1/10, and for each of these distributions, we run the EW $500$ times with $\eta_1=\eta_2=0.1$, stopping when the distance between $p^t$ and a strict  Nash equilibrium was  less than $0.0001$. 

The  following matrix, where the rows represent $p^0(T)$, and the columns represent $p^0(L)$, shows  the proportion of times the algorithm converged to $(T,L)$. 
We observe  that the values under the antidiagonal are all greater than $0.5$ while the values above are all smaller than 0.5.  This  perfectly make sense since  $f(x,y)=1/2$ whenever $x+y=1$, however we can also observe that $f(x,y)$ is not a function of the sum $x+y$.

\[\bordermatrix{
     &  0 & 0.1 & 0.2 & 0.3 & 0.4 & 0.5 & 0.6 & 0.7 & 0.8 & 0.9 & 1 \cr
 0   &  0 & 0   & 0   &  0  &  0  &  0  &  0  &  0  &  0  &  0  &0 \cr
 0.1 &  0 & 0 & 0 & 0 & 0 & 0.002 & 0.012 & 0.04 & 0.138 & 0.514 & 1 \cr
 0.2 &  0 & 0 & 0 & 0 & 0.008 & 0.022 & 0.102 & 0.242 & 0.512 & 0.844 & 1\cr
 0.3 &  0 & 0 & 0.002 & 0.01 & 0.052 & 0.128 & 0.264 & 0.478 & 0.748 & 0.97 & 1 \cr  
 0.4 &  0 & 0 & 0.008 & 0.04 & 0.178 & 0.3 & 0.5 & 0.702 & 0.904 & 0.99  & 1 \cr
 0.5 &  0 & 0.002 & 0.024 & 0.094 & 0.27 & 0.494 & 0.666 & 0.874 & 0.98 & 0.996 & 1 \cr 
 0.6 &  0 & 0.004 & 0.116 & 0.27 & 0.52 & 0.702 & 0.842 & 0.934 & 0.998 & 1 & 1 \cr 
 0.7 &  0 & 0.044 & 0.232 & 0.504 & 0.708 & 0.88 & 0.974 & 0.99 & 1 & 1 & 1 \cr 
 0.8 &  0 & 0.14 & 0.538 & 0.736 & 0.884 & 0.964 & 0.988 & 1 & 1 & 1 & 1 \cr
 0.9 &  0 & 0.504 & 0.884 & 0.97 & 0.992 & 1 & 1 & 1 & 1 & 1 & 1 \cr
 1   & 0 & 1 & 1 & 1 & 1 & 1 & 1 & 1 & 1 & 1 & 1 &
}
\]
\vspace{0.25cm}

\section{Open questions}
\label{sec:remarks}

We have shown that the Markov chain generated by the EW algorithm with constant learning rates converges almost surely to the set of strict Nash equilibria in the class of strong coordination games.
Outside of this class, we know that the only possible limit points are the Nash Equilibria with Equalizing Payoff and that  whenever a strict Nash equilibrium exists, then the probability to play this strict Nash equilibrium converges almost surely to 0 or 1. These results raise several interesting open questions and research directions for future works. 

\begin{enumerate}

\item What are the possible limits of the EW process?  We have shown in  Theorem \ref{thm3}   that any limit has to be a  Nash Equilibrium with Equalizing Payoff. However, in all our examples all limits of EW processes  are  NEEP with maximal support. Does this property generalize to any game ? In particular, can a pure, but not strict, Nash equilibrium be a limit of a EW process ?

\item We have shown that the almost sure convergence of the EW process is not guaranteed. For instance, it does not hold for zero-sum games like Matching Pennies, which do not have a Payoff Equalizing Nash Equilibria. As a consequence existence is not granted for generic games. An interesting case is the one of common payoffs games: does EW always converge for games with common payoffs ? We have checked this convergence holds  for  common payoff games with 2 players and 2 actions for each player (see the PhD thesis \cite{MDandrea}), but we do not know if the convergence extends to any number of actions and players. 

\item Computing the probability to converge to a given NEEP is difficult, it would be nice to have a non trivial example where an explicit formula can be obtained. More generally, in the class of strict coordination games, our results imply that there exists a random time $T$ such that the players always play one of the strict Nash equilibria after stage $T$. Is it possible to have any quantitative estimate to bound $T$? Is it true that $T$ is integrable?   \end{enumerate}
\section*{Acknowledgments}
We thank Mathieu Faure, Fr\'{e}d\'{e}ric Koessler, Panayotis Mertikopoulos, Laurent Miclo, Sylvain Sorin and Nicolas Vieille for useful comments.

\bibliographystyle{apalike} 
\bibliography{bibliografiaa}

\appendix
\section{Proofs}

\subsection{Proof of Lemma \ref{lem1}.}

$\;$

Denote by $f(p^0)$ the probability to always  play $a$ if we start at $p^0$.

a)  Using equalities (\ref{eq1}), we have: 
$$f(p^0)=\prod_{t=0}^{\infty}\prod_{i=1}^n \frac{1}{1+\sum\limits_{b_i\in A_i: b_i \neq a_i}\frac{  p^0_{i}(b_i)}{p^0_{i}(a_i)}e^{\eta_i t ( u_i(b_i,a_{-i})- u_i(a))}}.$$
 The infinite product is positive  if and only if:
\begin{equation}
\sum_{t=0}^{\infty}\sum_{i=1}^n \log\left(\frac{1}{1+\sum\limits_{b_i\in A_i: b_i \neq a_i}\frac{  p^0_{i}(b_i)}{p^0_{i}(a_i)}e^{\eta_i t ( u_i(b_i,a_{-i})- u_i(a))}}\right)>-\infty. \nonumber
\end{equation}
So our series converge if and only if
\begin{equation} \label{eq3}
    \sum_{t=0}^{\infty}\sum_{i=1}^n \log\left(1+\sum\limits_{b_i\in A_i: b_i \neq a_i}\frac{  p^0_{i}(b_i)}{p^0_{i}(a_i)}e^{\eta_i t ( u_i(b_i,a_{-i})- u_i(a))}\right)<\infty.
\end{equation}
Since 
$\lim_{t\to \infty}\sum_{b_i\in A_i: b_i \neq a_i}\frac{  p^0_{i}(b_i)}{p^0_{i}(a_i)}e^{\eta_i t ( u_i(b_i,a_{-i})- u_i(a))}=0,$
\noindent the convergence in (\ref{eq3}) holds if and only if
$$\sum_{t=0}^{\infty}\sum_{i=1}^n\sum_{b_i\in A_i: b_i \neq a_i}\frac{  p^0_{i}(b_i)}{p^0_{i}(a_i)}e^{\eta_i t ( u_i(b_i,a_{-i})- u_i(a))}<\infty.$$

This series is a finite sum of geometric series with argument less than 1 so it converges.
\vspace{0.25cm}

b)   Define $f_i(p_i^0)$ as the probability that Player $i$, starting from $p_i^0$, always plays $a_i$ if the other  players always play according to $a_{-i}$.
\begin{equation}\label{eq4}f_i(p^0_i)=\prod_{t=0}^{\infty} \frac{p_i^0(a_i)}{p_i^0(a_i)+\sum\limits_{b_i\in A_i: b_i \neq a_i} p^0_{i}(b_i)e^{\eta_i t ( u_i(b_i,a_{-i})- u_i(a))}}.\end{equation}
Since   $f(p^0)=\prod_{i\in N} f_i(p_i^0)$,  to prove that $f$ is continuous on $\Delta$,   it is  sufficient to show  that for every $i$, $f_i$   is continuous in $\Delta(A_i)$.
 
Fix a player $i$.   Let $\varepsilon>0$ and define $B=\{p^0_i\in\Delta(A_i)| \ p^0_i(a_i)\geq\varepsilon\}$. 
     For every $p^0_i\in B$, proceeding as in a) we have   $f^i(p^0_i)>0$.  Now, we can show that the infinite product in  (\ref{eq4}) is uniformly convergent in $B$. 
    Equivalently, we will show that       $$\sum_{t=0}^{\infty} \log\left(1+\sum\limits_{b_i\in A_i: b_i \neq a_i}\frac{  p^0_{i}(b_i)}{p^0_{i}(a_i)}e^{\eta_i t ( u_i(b_i,a_{-i})- u_i(a))}\right)$$
     is uniformly convergent in $B$. To prove the uniform convergence of the series is sufficient to bound the rest of the series by a sequence which goes to zero independently of $p^0_i$. Note that $p^0_i(b_i)/p^0_i(a_i)<1/\varepsilon$ for every $b_i\in A_i$ and every $p^0_i\in B$, so
       \begin{align}
      \sum_{t=T}^{\infty} &\log\left(1+\sum\limits_{\substack{b_i\in A_i:\\ b_i \neq a_i}}\frac{  p^0_{i}(b_i)}{p^0_{i}(a_i)}e^{\eta_i t ( u_i(b_i,a_{-i})- u_i(a))}\right) \nonumber\\
      \leq & \sum_{t=T}^{\infty} \log\left(1+\frac{1}{\varepsilon}\sum\limits_{\substack{b_i\in A_i:\\ b_i \neq a_i}}e^{\eta_i t ( u_i(b_i,a_{-i})- u_i(a))}\right) \nonumber \\
      \sim & \frac{1}{\varepsilon}\sum_{t=T}^{\infty} \sum\limits_{\substack{b_i\in A_i:\\ b_i \neq a_i}}e^{\eta_i t ( u_i(b_i,a_{-i})- u_i(a))}\xrightarrow[T\to\infty]{} 0. \nonumber
    \end{align}
  Hence, since the infinite product is uniformly convergent in $B$, the continuity of $f_i$ on $B$ follows.
  
  Now, to prove that $f_i$ is continuous in $\Delta(A_i)$ it remains  to show that it is continuous at points $p^*_i$, such that $p^*_i(a_i)=0$. For each $p_i^0$ we have: 
\[0\leq  f^i(p^0_i)=p^0_{i}(a_i)\prod_{t=1}^{\infty}\frac{p^0_{i}(a_i)}{p^0_{i}(a_i)+\sum\limits_{b_i\in A_i: b_i \neq a_i}p^0_{i}(b_i)e^{\eta_i t ( u_i(b_i,a_{-i})- u_i(a))}}\leq p^0_{i}(a_i),\]
 so $\lim_{p^0_i\to\eq_i}f_i(p^0_i)=f_i(p^*_i)=0.$ 
 \vspace{0.25cm}
 
 c)  Define  $\tau_n=\inf \{t \geq 0 , \|a-p^t\|\leq 1/n\}$ for each  positive  integer $n$.  Notice that  $(\tau_n< \infty) \supset (\tau_{n+1}<\infty)$. Plainly, $$(\exists t_0, \forall t\geq t_0, a^t=a) \subset (p^t\to a) \subset \cap_{m\geq 1} (\tau_m <\infty).$$
For each   $n$, 
 \begin{align}
  \mathbb{P}_{p^0}(\exists t_0, \forall t\geq t_0, a^t=a)   &=\P_{p^0}(\tau_n <\infty)\P_{p^0}(\exists t_0, \forall t\geq t_0, a^t=a  | \tau_n <\infty) \nonumber \\
    &\geq \P_{p^0}(\tau_n <\infty) f_n, \nonumber\\
    & \geq   \P_{p^0}(\cap_{m\geq 1} (\tau_m <\infty)) f_n, \nonumber \end{align}
where $f_n=  \min_{p, \|a-p\|\leq 1/n} f(p)$. By b),  $f_n$ increases to 1 as $n$ goes to infinity. So
$ \mathbb{P}_{p^0}(\exists t_0, \forall t\geq t_0, a^t=a) \geq  \P_{p^0}(\cap_{m\geq 1} (\tau_m <\infty))$, and since $(\exists t_0, \forall t\geq t_0, a^t=a) \subset (p^t\to a) \subset \cap_{m\geq 1} (\tau_m <\infty)$, we get almost surely:
\begin{equation}\label{eq0001}  \left(\cap_{m\geq 1} (\tau_m <\infty)\right)  \Longleftrightarrow(p^t\xrightarrow[t \to \infty]{}\  a)   \Longleftrightarrow \exists t_0 \; \forall t\geq t_0 \; a^t=a.\end{equation} This ends the proof of Lemma \ref{lem1}.\hfill \ensuremath{\square}

\subsection{Proof of Lemma \ref{lem2}.}

    Define as before, for every $p^0$ in $\Delta$: 
    $$f(p^0)=\mathbb{P}_{p^0}(a^t=a  \quad \forall t\geq 0)= \prod_{t=0}^{\infty}\prod_{i=1}^n\frac{1}{1+\sum\limits_{b_i\in A_i: b_i \neq a_i}\frac{  p^0_{i}(b_i)}{p^0_{i}(a_i)}e^{\eta_i t ( u_i(b_i,a_{-i})- u_i(a))}}\leq 1.$$

   Since $a$ is not a strict NE, there exists a player $i$ and an action $\hat{a}_i\neq a_i$ such that $u_i(\hat{a}_i,a_i)-u_i(a)\geq 0$. For $p^0$ in $\Int(\Delta)$, 
    $$\lim_{t\to\infty}\frac{p^0_i(\hat{a}_i)}{p^0_i(a_i)}e^{\eta_i t ( u_i(\hat{a}_i,a_{-i})- u_i(a))}=\begin{cases}
        \frac{p^0_i(\hat{a}_i)}{p^0_i(a_i)} &if \ u_i(\hat{a}_i,a_i)-u_i(a)=0 \\
        \infty & if \ u_i(\hat{a}_i,a_i)-u_i(a)>0.
    \end{cases}$$
    Hence when $t\mapsto \infty$, 
    $$\frac{1}{1+\sum\limits_{b_i\in A_i: b_i \neq a_i}\frac{  p^0_{i}(b_i)}{p^0_{i}(a_i)}e^{\eta_i t ( u_i(b_i,a_{-i})-u_i(a))}}\nrightarrow 1,$$
   so $f(p^0)=0$ for each $p^0$ in $\Int(\Delta)$.
   
  So  $\mathbb{P}_{p^0}(\exists t_0, \forall t\geq t_0, a^t=a  \quad \forall t\geq 0)\leq \sum_{t_0} \mathbb{P}_{p^0}( \forall t\geq t_0, a^t=a  \quad \forall t\geq 0)=0$.
  And $f$ is discontinuous since $f(a)=1$. \hfill\ensuremath{\square}
 
\subsection{Proof of Lemma \ref{lem25}.} $\;$

a) Assume $a\notin K$.

Define the events $H_n=\{\forall t\geq n, p^t \notin K\}$ for $n \geq 0$, and $H_\infty= \cup_n H_n=\{\exists n, \forall t\geq n, p^t\notin K\}$, and let $D=\min_{p\in K} \P_p(H_\infty)$.
We have $D\geq \min_{p \in K} \P_p(\forall t\geq 0, a^t=a)$ since $a\notin K$, and $\min_{p \in K} \P_p(\forall t\geq 0, a^t=a) >0$ by Lemma \ref{lem1}. So $D>0$, and we want to show that $D=1$. 

Fix $p^0\in K$. $\P_{p_0}(H_\infty)\geq D$ and since $H_n \subset H_{n+1}$ for each $n$, $\P_{p_0}(H_\infty)=\lim_{n\to \infty} \P_{p_0}(H_n)$. So there exists $n_0$ such that  $\P_{p_0}(H_{n_0})\geq D/2$. Define $\tau :=\inf\{ t\geq n_0: p^t\in K\}$, we have $H^{c}_{n_0}=\{\tau < \infty\}$.
Now,
\begin{align}
    \P_{p^0}(H_{\infty})&=\P_{p^0}(H_{n_0})\P_{p^0}(H_{\infty}| H_{n_0})+\P_{p_0}(\tau<\infty)\P_{p^0}(H_{\infty}|\tau < \infty) \nonumber \\
    &=\P_{p^0}(H_{n_0})+\E_{p^0}( \ind_{\tau<\infty}\E_{p^0}(\ind_{H_{\infty}}|\mathcal{F}_{\tau})) \nonumber \\
    &=\P_{p^0}(H_{n_0})+\E_{p^0}( \ind_{\tau<\infty}\P_{p^{\tau}}(H_{\infty})) \nonumber \\
    &\geq \P_{p^0}(H_{n_0})+(1-\P_{p^0}(H_{n_0}))D \nonumber\\
    &=D +\P_{p^0}(H_{n_0})(1-D)\nonumber \\
    &\geq D+\frac{D}{2}(1-D). \nonumber
\end{align}
As this is true for every $p^0 \in K$, we obtain $D\geq D+\frac{D}{2}(1-D)$, which together with the fact that $D\in (0,1]$ implies that $D=1$.
\vspace{0.25cm}

b) Assume now that $a\in K$. 

For each positive integer $n$, define  $K_n=\{p \in K, \|a-p\|\geq 1/n\}$. If $p^0 \neq a$, for $n$ large enough $p^0\in K_n$ so by a)  we have that almost surely  $\exists t_0 \; s.t.\;  \forall t \geq t_0,\ p^t\notin K_n$. So almost surely for each $n$ there exists $t_0$ such that for all $t\geq t_0$, $(p^t\notin K)$ or ($\|a-p^t\|<1/n$). 

We deduce that  a.s.:
\begin{equation}\label{eq001}(\exists t_0 \;, \forall t\geq t_0, p^t \notin K)\; {\rm  or}\;  (\forall n, \exists t\geq 0, \;\|p^t-a\|<1/n).\end{equation} 
Using Lemma \ref{lem1} c), we obtain: 
$$(\exists t_0 \;, \forall t\geq t_0, p^t \notin K)\; {\rm  or}\; (\exists t_0, \forall t\geq t_0, a^t=a).$$
This ends the proof of Lemma \ref{lem25}.\hfill \ensuremath{\square}

\subsection{Proof of Lemma \ref{lem61}.}

{From the preceding discussion, $(p_t)$ converges almost surely to some limit $p_\infty$. With probability $1$, $p_\infty$ belongs to ${\cal P}$ (by Theorem \ref{thm3}) and cannot be $(T,R)$ or $(B,L)$. It remains to show that $p_\infty$ cannot be equal to $(T,L)$ with positive probability.} 

Fix  $p^0$ in $\Int(\Delta)$ and write $x=p_1^0(T)$, $y=p_2^0(L)$. 

{
We will show that $\P_{p^0}(H_\infty)=1$ where 
\[ H_\infty= (\exists t \geq 0, \forall s \geq t, \; a^1_s=T) \cup (\exists t \geq 0, \forall s \geq t, \; a^2_s=L),\]
implying that almost surely either $(p^1_t)_t$ or $(p^2_t)_t$ is eventually constant and therefore that $p_\infty$ has maximal support in ${\cal P}$. }

{
Define for all $t\geq 0$:
\[ H_t= ( \forall s \geq t, \; a^1_s=T) \cup (\forall s \geq t,\; a^2_s=L),\]
so that $H_t \subset H_{t+1}$ and $H_\infty= \cup_{t\geq 0}H_t$.
}

The process $(Z_t)_t:= (B_t,R_t)_t$ is a Markov chain with values in $\N^2$ and transitions such that for each $t\geq 0$, $B_{t+1}$ and $R_{t+1}$ are {conditionnally} independent given $Z_t$,  and:
\[\P(B_{t+1}=B_t| Z_t)=\frac{x}{x+(1-x)e^{-\eta_1 R_t}}\; , \; \P(B_{t+1}=1+B_t| Z_t)=1- \P(B_{t+1}=B_t| Z_t),\]
\[\P(R_{t+1}=R_t| Z_t)=\frac{y}{y+(1-y)e^{-\eta_2 B_t}}\; , \; \P(R_{t+1}=1+R_t| Z_t)=1- \P(R_{t+1}=R_t| Z_t).\]
The sequences $(B_t)_t$ and $(R_t)_t$ are weakly increasing and thus converge respectively  in $\N\cup\{\infty\}$ to $B_\infty$ and $R_\infty$. Plainly, $B_\infty+R_\infty=+\infty$ { almost surely, as the probability that the chain never moves after reaching any position $Z_t$ is zero.}

We make a change of time variable and only consider the stages  where $B_t+R_t$ strictly increases. Define $t_0=0$ and for each $s\geq 1$, $t_s=\min\{t>t_{s-1}, B_t+R_t>B_{t-1}+R_{t-1}\}.$ Write for each $s\geq 0$, $B'_s=B_{t_s}$ and $R'_s=R_{t_s}$. The process $(Z'_s)_s:= (B'_s,R'_s)_s$ is well-defined with probability $1$, and is a Markov chain with values in $\N^2$ and transitions given by:

\begin{align*}
\P(B'_{s+1}=1+B'_s, & R'_{s+1}=R'_s| Z'_s)=\frac{\frac{(1-x)e^{-\eta_1 R'_s}}{x+(1-x)e^{-\eta_1 R'_s}}\frac{y}{y+(1-y)e^{-\eta_2 B'_s}}}{1-\frac{xy}{(x+(1-x)e^{-\eta_1 R'_s})(y+(1-y)e^{-\eta_2 B'_s})}} \\
&=\frac{y(1-x) e^{-\eta_1R'_s}}{y(1-x)e^{-\eta_1R'_s}+x(1-y)e^{-\eta_2B'_S}+(1-x)(1-y)e^{-\eta_1R'_S-\eta_2B'_S}},
\end{align*}
\begin{align*}
\P(B'_{s+1}&= B'_s, R'_{s+1}=1+R'_s| Z'_s)\\
&=\frac{x(1-y) e^{-\eta_2B'_s}}{y(1-x)e^{-\eta_1R'_s}+x(1-y)e^{-\eta_2B'_S}+(1-x)(1-y)e^{-\eta_1R'_S-\eta_2B'_S}},
\end{align*}
\begin{align*}
\P(B'_{s+1}&= 1+B'_s, R'_{s+1}=1+R'_s| Z'_s)\\
&=\frac{(1-x)(1-y)e^{-\eta_1R'_S-\eta_2B'_S}}{y(1-x)e^{-\eta_1R'_s}+x(1-y)e^{-\eta_2B'_S}+(1-x)(1-y)e^{-\eta_1R'_S-\eta_2B'_S}}.
\end{align*}
\noindent We have
\begin{align*}
\P(Z'_{s+1}=Z'_s+(1,0)| Z'_s)  & =   \frac{y(1-x) e^{-\eta_1R'_s}}{y(1-x)e^{-\eta_1R'_s}+x(1-y)e^{-\eta_2B'_s}+(1-x)(1-y)e^{-\eta_1R'_s-\eta_2B'_s}},\\
& =   \frac{1}{1+\frac{(1-y)}{y}e^{-\eta_2B'_s}(1+\frac{x}{1-x}e^{\eta_1R'_s})}
 \end{align*}
Consequently, using the Markov property
\begin{align*}
\P_{p^0}(\forall t \geq 0, a^2_t=L) &= \P_{p^0}(\forall t\geq 0, Z'_{t+1}= Z'_t+(1,0)) \\
& =  \prod_{t\geq 0} \frac{1}{1+\frac{(1-y)}{y(1-x)}e^{-\eta_2 t}}.
\end{align*}
Note that this infinite product is positive,  non-decreasing with respect to $y$ and non-increasing with respect to $x$. 
Similarly, we have
 \begin{align*}
\P_{p^0}(\forall t \geq 0, a^1_t=T) &= \P_{p^0}(\forall t\geq 0, Z'_{t+1}= Z'_t+(0,1)) \\
& = \prod_{t\geq 0} \frac{1}{1+\frac{(1-x)}{x(1-y)}e^{-\eta_1 t}}.
\end{align*}
We deduce that
\begin{align*}
\P_{p^0}(H_\infty) &\geq \max\{\P_{p^0}(\forall t \geq 0, a^1_t=T),\P_{p^0}(\forall t \geq 0, a^2_t=L)\} \\
 & \geq  \max\left\{\prod_{t\geq 0} \frac{1}{1+\frac{(1-y)}{y(1-x)}e^{-\eta_2 t}},\prod_{t\geq 0} \frac{1}{1+\frac{(1-x)}{x(1-y)}e^{-\eta_1 t}} \right\}\\
& \geq  \prod_{t\geq 0} \frac{1}{1+\frac{1}{\min\{x,y\}}e^{-\min\{\eta_1,\eta_2\} t}}>0.
\end{align*}
We can now conclude as in Lemma \ref{lem25}. Define $E=\{ p \in \Delta \,:\,p_1(T)\geq x, p_2(L)\geq y\}$ and 
\[D= \inf_{p \in E}\P_{p}(H_{\infty}) .\]
\noindent Note that for every $p \in E$, we have 
$$\P_{p}(H_{\infty}) \geq \prod_{t\geq 0} \frac{1}{1+\frac{1}{\min\{p_1(T),p_2(L)\}}e^{-\min\{\eta_1,\eta_2\} t}}\geq \prod_{t\geq 0} \frac{1}{1+\frac{1}{\min\{x,y\}}e^{-\min\{\eta_1,\eta_2\} t}} >0.$$
so that $D>0$. Let $p \in E$. Since $\P_{p}(H_\infty)=\lim_{n\to \infty} \P_{p}(H_n)$, there exists $n_0$ such that  $\P_{p}(H_{n_0})\geq D/2$. Define $\tau :=\inf\{ t\geq n_0: B_t>B_{n_0}\text{ and } R_t>R_{n_0}  \}$, so that $H^{c}_{n_0}=\{\tau < \infty\}$.
We have
\begin{align}
    \P_{p}(H_{\infty})&=\P_{p}(H_{n_0})\P_{p}(H_{\infty}| H_{n_0})+\P_{p}(\tau<\infty)\P_{p}(H_{\infty}|\tau < \infty) \nonumber \\
    &=\P_{p}(H_{n_0})+\E_{p}( \ind_{\tau<\infty}\E_{p}(\ind_{H_{\infty}}|\mathcal{F}_{\tau})) \nonumber \\
    &=\P_{p}(H_{n_0})+\E_{p}( \ind_{\tau<\infty}\P_{p^{\tau}}(H_{\infty})) \nonumber \\
    &\geq \P_{p}(H_{n_0})+(1-\P_{p}(H_{n_0}))D \nonumber\\
    &=D +\P_{p}(H_{n_0})(1-D)\nonumber \\
    &\geq D+\frac{D}{2}(1-D), \nonumber
\end{align}
As it is true for all $p\in E$, we obtain $D\geq D+\frac{D}{2}(1-D)$, which together with the fact that $D\in(0,1]$ implies that $D=1$. \hfill $\Box$  
\end{document}